\pdfoutput=1
\RequirePackage{fix-cm}
\documentclass{article}
\usepackage[T1]{fontenc}
\usepackage[utf8]{inputenc}
\usepackage{textcomp}
\usepackage{arxiv}
\usepackage{amsmath,amssymb,amsthm}
\usepackage{graphicx}
\usepackage{booktabs,array,longtable}
\usepackage[labelfont=bf,labelsep=period,font=normalsize,justification=justified,singlelinecheck=false]{caption}
\usepackage{xcolor}
\usepackage{fancyvrb}
\usepackage{enumitem}
\setlist{topsep=2pt,itemsep=1pt,parsep=1pt}
\usepackage[round,semicolon]{natbib}
\defcitealias{easa2024conceptpaper}{EASA, 2024}
\defcitealias{nist2023airmf}{NIST, 2023}
\defcitealias{faa2017ac20115d}{FAA, 2017}
\defcitealias{milhdbk5j2003}{MIL-HDBK-5J, 2003}
\usepackage[hidelinks]{hyperref}
\newcommand{\procfont}{\ttfamily\fontsize{9.5}{11.4}\selectfont}
\newcommand{\monocell}[1]{\makebox[0.525em][c]{#1}}
\DeclareUnicodeCharacter{03BA}{\monocell{$\kappa$}}
\DeclareUnicodeCharacter{015C}{\monocell{$\hat{S}$}}
\DeclareUnicodeCharacter{03C4}{\monocell{$\tau$}}
\DeclareUnicodeCharacter{2190}{\monocell{$\leftarrow$}}
\DeclareUnicodeCharacter{2192}{\monocell{$\rightarrow$}}
\DeclareUnicodeCharacter{2208}{\monocell{$\in$}}
\counterwithin{figure}{section}
\counterwithin{table}{section}
\newtheoremstyle{arxivdef}{6pt}{6pt}{\normalfont}{}{\bfseries}{.}{0.5em}{\thmname{#1}\thmnumber{ #2}\thmnote{ (#3)}}
\theoremstyle{arxivdef}
\newtheorem{definition}{Definition}
\newtheorem{proposition}{Proposition}
\newtheorem{corollary}{Corollary}
\newcounter{procedure}
\usepackage[section]{placeins}
\usepackage{flafter}
\usepackage{needspace}

\newtheorem*{assumption*}{Assumption}
\newtheorem*{remark*}{Remark}
\hypersetup{pdftitle={Dependable AI-Assisted Engineering: A Formal Framework for AI Participation and Assurance in Safety-Critical Workflows}, pdfauthor={Puxue Tan},
  pdfkeywords={AI-assisted engineering, generative AI, engineering assurance, safety-critical workflows, formalization fidelity, human-in-the-loop engineering}}
\title{Dependable AI-Assisted Engineering: A Formal Framework for AI Participation and Assurance in Safety-Critical Workflows}
\renewcommand{\shorttitle}{Dependable AI-Assisted Engineering}
\date{}
\author{Puxue Tan\\
  School of Mechanical and Aerospace Engineering\\ Queen's University Belfast\\ 50 Malone Road, Belfast, UK\\
  \texttt{p.tan@qub.ac.uk}}
\begin{document}
\maketitle
\begin{abstract}
Generative AI can produce engineering artefacts, but generation alone does not determine whether or how those artefacts should enter safety-critical workflows. This paper develops a formal framework for assigning AI participation and assurance at the level of individual workflow units. Each unit has a participation and assurance record covering its engineering requirement, an approved operational formalization where applicable, the applicable mechanism, fallback where applicable, evidence obligations and the applicable guarantee, plus a deployment-readiness status. The framework distinguishes deterministic verification, statistically calibrated admission, authorized human judgement supported by AI advice, authorized human adjudication of AI-produced artefacts, retained deterministic tool paths and explicit non-participation; these arrangements carry different kinds of guarantee rather than levels on a common scale. The framework also separates formalization fidelity from verifier soundness, provides a staged classification and readiness procedure, and derives conditions for comparing a gated AI-assisted unit with an incumbent process under recurring-population assumptions. We instantiate and apply the framework in an executed 17-unit wing-spar structural-analysis workflow combining deterministically gated AI-generated CAD, retained deterministic computation and human judgement. The AI-generated CAD program passed all 23 deterministic checks and was admitted at the first attempt. Favourable stress magnitudes did not suffice to pass the stress criteria where the predeclared mesh-convergence evidence was insufficient; those criteria were instead referred to engineering judgement. The case demonstrates selective AI participation and explicit evidence handling at unit level; no claim is made of workflow-level dependability, certification, structural safety or productivity.
\end{abstract}
\keywords{AI-assisted engineering \and generative AI \and engineering assurance \and safety-critical workflows \and formalization fidelity \and human-in-the-loop engineering}

\section{Introduction}\label{sec:1}

\subsection{AI participation in safety-critical engineering workflows}\label{sec:1.1}

Generative AI systems increasingly produce engineering artefacts, from analysis scripts and parametric CAD programs to technical documentation. In safety-critical engineering, however, assurance rests on traceable and reviewable evidence relating decisions and deliverables to their requirements. Introducing AI changes who or what produces a deliverable; it does not remove the need for that evidence.

The generation process does not, by itself, establish that an output is acceptable for engineering use. Whether the generator is deterministic or stochastic, capable or error-prone, acceptance depends on what generation does not supply: the requirement the artefact must satisfy, the way that requirement is made operational for checking, the rule by which the artefact is admitted, and the course taken when the evidence is insufficient. Different parts of the same workflow may therefore call for different forms of AI participation and assurance, and the same AI system may suit one part of a workflow but not another.

Placing human review around every AI contribution is a conservative response, but not a complete organizing principle. Authorized human judgement remains necessary where it is the appropriate acceptance mechanism. Elsewhere, a repeatable deterministic check, a calibrated statistical gate or an established deterministic tool may be more appropriate, and some parts of a workflow should involve no AI at all. Uniform review also does not, by itself, record what was checked, against which specification and by whom. The difficulty is not review effort as such, but the absence of an explicit rule for how AI participates in each unit of engineering work and which assurance mechanism governs it.

The motivating concern is dependability, the ability to deliver service that can justifiably be trusted \citep{avizienis2004dependability}. In engineering workflows, that trust rests on evidence attached to individual contributions. The relevant question is therefore not whether AI is globally trustworthy, but, for each unit of work, how AI participates, what assurance that participation requires, and what happens when the required evidence is insufficient. Where an AI-produced artefact is proposed for release, insufficient evidence should lead to rejection, a declared fallback or authorized adjudication rather than implicit acceptance.

\subsection{Beyond binary trust}\label{sec:1.2}

Framing AI adoption as a binary trust decision (whether an AI output can be trusted or not) obscures the fact that different parts of a workflow admit different forms of AI participation, evidence and acceptance. Existing work addresses complementary parts of this problem, from aviation guidance on human–AI authority and learning assurance \citepalias{easa2024conceptpaper} to assurance cases, formal verification, statistical risk control, human oversight and tool qualification (Section~\ref{sec:2}).

The closest conceptual prior work is the architectural compensation discussed by \citet{bloomfield2025aiassurance}, in which strongly assured guards or diverse secondaries compensate for components that can be only weakly assured. The present work shares that concern with controlling weakly assured components. It formalizes how AI participation is assigned within a workflow and what assurance supports it, making the applicable mechanism, fallback, evidence obligations and readiness explicit.

In the sources reviewed, we did not identify a scheme that combines four elements for each workflow unit in an engineering workflow:

\begin{itemize}
\item a decision on whether AI participates, and under what acceptance mechanism an AI-produced artefact may be admitted;
\item a declared course of action when it is not admitted: fallback, rejection or authorized adjudication;
\item an explicit assurance state, consisting of evidence obligations with explicit evidence kinds and a readiness status kept distinct from the unit's classification;
\item explicit treatment of the step from requirement to operational formalization.
\end{itemize}

The gap concerns their integration at the level of individual workflow units, rather than any one element in isolation.

Engineering contrasts show why these elements must be decided unit by unit. An AI-generated CAD program may admit deterministic checks against an approved operational formalization of its geometry, provided the checker's soundness is supported by evidence. Such checking does not, by itself, establish fidelity to the full informal requirement. A stress calculation may already have an established deterministic tool path, and there the question is whether AI should be inserted at all. Concept selection or the acceptance of an unresolved finding may depend on authorized engineering judgement. Where a recurring population, a fixed generator configuration and trusted reference labels exist, calibrated statistical admission becomes a further option. These are different participation and assurance arrangements, not points on a single assurance scale; what unifies them is an explicit rule for how AI participates and what assurance that participation requires. Section~\ref{sec:1.3} states the research questions this raises and the paper's contributions.

\subsection{Contributions and paper organization}\label{sec:1.3}

This paper addresses three linked questions:

\begin{itemize}
\item how AI participation should be allocated across the units of work in an engineering workflow;
\item under what acceptance mechanism an AI-produced artefact may be admitted as a unit's deliverable;
\item how evidence that does not support acceptance should be handled.
\end{itemize}

It makes three contributions.

\textbf{1. A formal framework for AI participation and assurance.} For each relevant unit of work, the framework records the engineering requirement; an approved operational formalization of a declared scope and version, where applicable; the assigned participation and assurance regime; the applicable acceptance or execution mechanism; any declared fallback; the assurance obligations and their evidence; deployment readiness; and the regime guarantee.

The regimes, covering deterministic and statistically calibrated admission, human judgement supported by AI advice, human adjudication of AI-produced artefacts, retained deterministic tool paths and explicit non-participation, differ in participation arrangement, assurance mechanism and guarantee kind, with no scalar dependability score or total ordering. The framework also treats the translation from requirement to operational formalization as an explicit assurance problem: approving a formalization is a recorded act, not a proof of its fidelity.

\textbf{2. A staged classification and readiness procedure.} As part of the formal framework, the procedure separates participation (whether AI should take part in a unit of work at all), assurance (which assurance mechanism is actually available) and readiness (whether the applicable obligations are currently evidenced). It records the current regime, any target regime attainable in principle and its deployment readiness as distinct statuses. Deterministic and statistical admission are assigned only where the verification or calibration mechanism actually exists. Its decision predicates, such as whether AI adds justified value, whether delegation is permitted and whether an approved formalization exists, remain engineering and organizational judgements, recorded with their justification.

\textbf{3. An executed engineering demonstration.} The framework was instantiated and applied to a structural-analysis workflow for a tapered, stiffened wing-spar segment, executed with regimes, criteria and tools fixed before the run. An AI-generated CAD program entered the analysis only through deterministic admission, established deterministic tools were retained, AI-produced drafts were released only through authorized human acceptance, and engineering judgement, supported by AI advice, was applied explicitly where the evidence did not resolve a criterion. Favourable numerical values did not yield a pass where the required convergence evidence was absent, even though every unit's evidence obligations had been satisfied before the run. The study is author-led and bounded in scope; its limitations are set out in Section~\ref{sec:6.2}.

Sections~\ref{sec:2}--\ref{sec:7} present related work, the framework, the classification procedure, the case study, the discussion and the conclusions. The Supplement (Sections S1–S7) provides the case-study evidence and provenance, together with a retrospective audit of the original case-study record, made before the controlled execution.

\section{Background and Related Work}\label{sec:2}

\subsection{AI-generated engineering artefacts}\label{sec:2.1}

Generative models now produce engineering artefacts directly. DeepCAD \citep{wu2021deepcad} generates CAD operation sequences with a Transformer and is evaluated by reconstruction accuracy and the ratio of invalid outputs; CAD-MLLM \citep{xu2024cadmllm} conditions parametric CAD generation on text, images or point clouds and adds topology-quality metrics. \citet{kumar2025generative} use a large language model to write FreeCAD scripts, refined from execution errors and evaluated mainly by whether they run, and \citet{zhang2026llmcad} survey large language models for computer-aided design and the metrics in current use.

These studies demonstrate generation capability and measure it with geometric, topological and executability metrics. They do not, by themselves, establish the rule under which a generated artefact should be admitted into an engineering workflow, nor what should happen when it is not admitted.

\subsection{Assurance, governance and autonomy-level approaches}\label{sec:2.2}

Aviation guidance addresses AI primarily at the level of the application and its machine-learning constituents. EASA's AI Concept Paper Issue 02 \citepalias{easa2024conceptpaper} classifies applications as Level 1 (“assistance to human”) or Level 2 (“human-AI teaming”), with sublevels distinguished by the function allocated to the AI-based system and the authority retained by the end user, and structures learning assurance as a W-shaped process. The Overarching Properties (Intent, Correctness and Innocuity) are properties of a product for which approval is sought. They emerged from an FAA-sponsored initiative, and the report presenting them is not itself official guidance \citep{holloway2019overarching}. The joint EUROCAE WG-114/SAE G-34 standard ED-324/ARP6983, still a draft, defines a machine-learning development lifecycle for ML constituents \citetext{\citeauthor{eurocae_ed324_draft}}. \citet{luettig2024mlaerospace} classify aerospace AI applications by criticality and AI-method complexity.

At the organizational level, the NIST AI Risk Management Framework \citepalias{nist2023airmf} organizes AI risk management into four functions, and ISO/IEC 42001:2023 \citep{iso42001_2023} specifies requirements for an AI management system.

Assurance-case approaches structure arguments and evidence. AMLAS \citep{hawkins2021amlas} provides a stage-based safety-case process for machine-learning components. Assurance 2.0 \citep{bloomfield2021assurance2} strengthens claims–arguments–evidence cases. \citet{lee2026safetycase} propose reusable safety-case templates together with taxonomies of claims, arguments and seven evidence families. Surveys organize the wider ML-safety and ML-certification literature \citep{mohseni2022mlsafety,tambon2022certify}. The approaches surveyed here answer different questions at different levels, and no correspondence between them and the participation and assurance regimes of Section~\ref{sec:3} is claimed.

\textbf{Closest prior work.} \citet{bloomfield2025aiassurance} observe that where only weak assurance can be provided for a component, such as a machine-learning component, the system architecture should compensate, for example with a strongly assured guard performing runtime checking or a diverse secondary. This is closely related to the present concern with how AI outputs enter an engineering process; the present framework assigns participation and assurance unit by unit within an engineering workflow (Section~\ref{sec:1.2}).

\subsection{Deterministic verification and formalization fidelity}\label{sec:2.3}

Formalized properties can in many settings be checked deterministically. SMT solvers such as Z3 \citep{demoura2008z3} support automated checking over formalized constraints in supported logical theories, and SAT-solver unsatisfiability verdicts, for example, can be accompanied by proof certificates checked by independent, in some cases formally verified, checkers \citep{cruzfilipe2017lrat}. Two distinctions matter here.

\textbf{Checking is distinct from generating.} A deterministic checker can decide whether an artefact satisfies a formalized property irrespective of how it was produced. The strength of that decision rests on the checker's soundness, which is supported by proof where available and by engineering evidence otherwise.

\textbf{Soundness relative to a formalization is distinct from fidelity to the requirement.} \citet{zave1997darkcorners} argue that a specification satisfies the requirements only in conjunction with domain knowledge, so meeting a specification does not by itself establish the requirement. The same holds for the chain from requirement to approved operational formalization to verifier: a verifier may be sound with respect to its formalization while the formalization fails to capture part of the informal requirement.

Tools and practices support fidelity without proving it: NASA's FRET translates structured natural-language requirements into temporal logic \citep{giannakopoulou2020fret}, and requirements traceability has long been recognized as a persistent problem \citep{gotel1994traceability}. Formalization fidelity therefore remains an assurance and governance problem in its own right.

\subsection{Statistical acceptance and selective risk}\label{sec:2.4}

Statistical admission builds on selective prediction, in which a predictor may abstain; \citet{chow1970reject} characterized the trade-off between error and rejection. In selective classification, coverage is the fraction of inputs on which a classifier predicts, and selective risk is its error among the inputs it accepts \citep{elyaniv2010selective}. \citet{geifman2017selective} choose an acceptance threshold so that selective risk does not exceed a target with probability at least 1 $-$ $\delta$ over i.i.d. samples, while making coverage as large as possible.

Risk-controlling calibration generalizes this idea. \citet{bates2021rcps} bound the expected loss of a black-box predictor's set-valued output at level $\alpha$ with probability at least 1 $-$ $\delta$. Learn then Test \citep{angelopoulos2025ltt} controls user-chosen risks, including selective error among non-abstained predictions, at stated confidence.

Such guarantees depend on stated sampling assumptions and on the predictor or generator configuration to which the calibration applies, and they carry a sample-size cost: for a fixed acceptance rule with no observed violations among independent trials, exact binomial limits give the familiar zero-numerator bound of roughly 3/n at one-sided 95 \% confidence for n trials \citep{clopper1934binomial,hanley1983zero}.

In the sources reviewed, coverage is maximized or reported rather than guaranteed jointly with selective risk at stated confidence. The joint requirement used in Section~\ref{sec:3}, bounded selective risk together with a minimum coverage at stated confidence, is this paper's own construction on these foundations.

Agreement among models does not itself provide calibrated risk control. \citet{kim2025correlated} find substantially correlated errors across more than 350 language models, with larger, more accurate models showing highly correlated errors even across providers and architectures. \citet{goel2025great} observe that model mistakes become more similar as capability increases. Taken together, these findings motivate treating agreement among models as a signal whose relation to correctness must be calibrated against the relevant formalization, rather than as a correctness guarantee.

\subsection{Human oversight, established tools and model adequacy}\label{sec:2.5}

Where human oversight is used as an acceptance mechanism, its effectiveness depends on how it is exercised. \citet{parasuraman1997humans} distinguish use, misuse (over-reliance), disuse and abuse of automation. \citet{parasuraman2010complacency} review evidence that automation complacency and automation bias affect naive and expert users alike, and that automation bias produces errors of omission and commission that training alone does not prevent. \citet{bansal2021whole} report that AI explanations increased acceptance of AI recommendations regardless of their correctness. Human authority over a decision matters, but recorded human acceptance does not by itself establish that the accepted artefact is correct.

Established deterministic tools raise a different boundary. A tool path with fixed versions, inputs and configuration makes outputs repeatable, but repeatability is not validation, and validation evidence is not, by itself, formal qualification. In airborne software, tool qualification is a defined process (RTCA DO-330/EUROCAE ED-215 \citep{rtca2011do330}, recognized by FAA Advisory Circular 20-115D \citepalias{faa2017ac20115d}) whose required qualification level depends on the tool's potential impact on the software life cycle and on the software level of the item it supports. Whether a computational model is adequate for its intended use is a separate question, addressed by verification and validation practice such as ASME V\&V 10-2019 \citep{asme2019vv10}, a question that retaining an established tool does not settle.

\subsection{Synthesis and research gap}\label{sec:2.6}

The literature offers mature, complementary mechanisms for artefact generation, assurance of machine-learning constituents, organizational risk management, assurance cases, deterministic and proof-checked verification, statistical risk control, the human factors of oversight, and tool qualification and model verification and validation. These address different questions about different objects: the application, the organization, the assurance argument, the artefact or the tool. Read together, they also mark the limits of each mechanism: deterministic checking is sound only relative to a formalization whose fidelity it does not establish (Section~\ref{sec:2.3}); statistical admission requires stated sampling assumptions and a fixed configuration (Section~\ref{sec:2.4}); recorded human acceptance confers authority but not correctness; and a retained tool secures repeatability, not model adequacy (Section~\ref{sec:2.5}). In the sources reviewed, however, we did not identify a scheme that integrates these mechanisms, and their limits, for each workflow unit in an engineering workflow by combining the four elements set out in Section~\ref{sec:1.2}.

Dependability motivates this integration but is not formalized here as a scalar property. Section~\ref{sec:3} formalizes the integration as a framework for AI participation and assurance.

\section{Formal Framework}\label{sec:3}

The framework's unifying principle is \textbf{an explicit AI-participation decision for each unit of engineering work, paired with the assurance required to justify that decision}. Where AI produces an artefact for possible release (Types A, B and \texttt{C[adjudicate]}), an acceptance or admission mechanism is part of that assurance; under \texttt{C[assist]}, AI is advisory and an authorized person produces the deliverable; Type D retains an established deterministic path; and Type E records justified non-participation. The regimes defined below differ in participation arrangement, assurance mechanism and guarantee kind, and share no common scalar measure; dependability (Section~\ref{sec:1.1}) is the motivating system-level property.

\subsection{Activities, Workflow Units, Requirements and Formalizations}\label{sec:3.1}

An engineering \textbf{workflow} $W$ is a finite collection of \textbf{activities}: recognizable units of engineering work such as geometry definition, mesh generation or compliance checking. Activities are the units in which engineers describe their work, but not necessarily the units at which assurance can be meaningfully assigned: a single activity may contain parts whose acceptance can be decided by computation alongside parts that require professional judgement. The framework therefore classifies \textbf{workflow units}.

\begin{definition}[Workflow unit]\label{def:1}

A \emph{workflow unit} $\kappa$ of an activity is a unit of work with:

\begin{itemize}
\item an input space $I_\kappa$;
\item an output (artefact) space $O_\kappa$;
\item a declared \textbf{valid-input domain} $I^{valid}_\kappa\subseteq I_\kappa$;
\end{itemize}

such that its AI-participation and assurance characteristics are sufficiently homogeneous for a single participation and assurance regime to be meaningful. An activity is decomposed into one or more workflow units (Section~\ref{sec:4.3}).
\end{definition}

The valid-input domain records the conditions that inputs are assumed to satisfy: units, provenance, consistency with upstream deliverables, and so on. \textbf{Every guarantee in this framework is conditional on $i\in I^{valid}_\kappa$.} Whether upstream units actually deliver such inputs is outside the scope of a single assurance record (Definition~\ref{def:6}; Section~\ref{sec:4.4}).

\begin{definition}[Requirement and approved formalization]\label{def:2}\leavevmode

\begin{itemize}
\item The \textbf{requirement} $R_\kappa$ of a workflow unit is the condition its deliverable must satisfy, as stated by the relevant regulation, organization, engineering practice or design intent. $R_\kappa$ is generally expressed in natural language. It need not be computable. It need not even be bivalent: for some units, competent experts may legitimately disagree on whether a given output meets it.
\item An \textbf{approved formalization} of $R_\kappa$, where one exists, is a predicate
\[
\hat S_\kappa:\ I_\kappa\times O_\kappa\to\{0,1\},
\]
together with a declared \textbf{scope} (the part of $R_\kappa$ it is intended to capture) and a \textbf{version identifier}, adopted through a recorded approval.
\end{itemize}
\end{definition}

In general $R_\kappa\neq\hat S_\kappa$: the formalization is an operational interpretation of a declared scope of the requirement, adopted for the purpose of checking, not the requirement itself. Whether, for example, inequalities over computed stresses faithfully represent the governing regulation is addressed in Section~\ref{sec:3.8}.

\begin{definition}[Generator]\label{def:3}

A \emph{generator} $G_\kappa$ is an AI system that, given $i\in I_\kappa$, produces an artefact $o\sim G_\kappa(i)$. The generator is treated as a black box and may be stochastic. When a guarantee depends on the generator's behaviour (Type B), the generator is identified by a \textbf{version identifier}, and the guarantee refers to that version only.
\end{definition}

\subsection{Participation and Assurance Records}\label{sec:3.2}

\subsubsection*{Gate, fallback, and released output}

\begin{definition}[Gate, admission event, fallback]\label{def:4}\leavevmode

\begin{itemize}
\item A \emph{gate} for workflow unit $\kappa$ is a rule that decides, for an input $i$ and an AI-produced artefact $o$, whether $o$ is \textbf{admitted} as the unit's deliverable. We write $E$ for the admission event.
\item The \emph{fallback} $F_\kappa$ is the declared path to which execution is routed when the artefact is not admitted. Examples are regeneration, manual execution, or escalation to human adjudication.
\item The \emph{released output} is
\[
O_{rel}(i)=\begin{cases}o & \text{on }E,\\ F_\kappa(i) & \text{on }\neg E,\end{cases}
\]
where $F_\kappa(i)$ may be empty ($\bot$) if the fallback does not produce a deliverable.
\end{itemize}
\end{definition}

Rejection routes execution to the fallback; it does not by itself ensure that a deliverable is produced, or that the fallback's deliverable is correct. The availability and quality of the fallback are properties of the fallback path, to be evidenced separately where they matter.

\subsubsection*{Evidence obligations}

A regime's guarantee holds only under conditions the framework cannot establish on its own, such as verifier soundness, representative calibration data, human authority or use of a tool within its validity domain. These are recorded as \textbf{assurance obligations} with explicit evidence kinds.

\begin{definition}[Assurance obligation]\label{def:5}

An \emph{assurance obligation} is a condition on which a regime's guarantee depends, recorded with its supporting evidence and one of four evidence kinds:

\begin{itemize}
\item \textbf{[Proof]:} a mathematical proof obligation, \emph{discharged} by a proof (e.g., a checked proof certificate or a verified checker);
\item \textbf{[Emp]:} an empirically validated condition, supported by data and carrying statistical uncertainty;
\item \textbf{[Org]:} an organizationally authorized condition, established by a recorded decision of a person or body with the relevant authority;
\item \textbf{[Eng]:} an engineering assumption supported by evidence (review, independent implementation, testing, validation history) but not mathematically proven.
\end{itemize}
\end{definition}

Only [Proof] obligations are \emph{discharged}; the others are \emph{satisfied with stated evidence} (testing a verifier on examples supports but does not prove its soundness). Whether a record's obligations are satisfied determines its \emph{deployment readiness} (Definition~\ref{def:6}), kept separate from which regime a unit requires (Section~\ref{sec:4.2}).

\subsubsection*{The record}

\begin{definition}[Participation and assurance record]\label{def:6}

A \emph{participation and assurance record} (for short, \emph{assurance record}) for workflow unit $\kappa$, denoted $\mathsf{AR}_\kappa$, is
\[
\mathsf{AR}_\kappa=\big\langle\,R_\kappa,\ \hat S_\kappa,\ r_\kappa,\ \mathrm{Mech}_\kappa,\ F_\kappa,\ \mathrm{Obl}_\kappa,\ G_{r_\kappa}\,\big\rangle,
\]
with a version identifier. Its fields are:

\begin{itemize}
\item $R_\kappa$: the requirement.
\item $\hat S_\kappa$: the approved formalization and its scope, where one exists.
\item $r_\kappa\in\{A,\ B,\ C[\text{assist}],\ C[\text{adjudicate}],\ D,\ E\}$: the participation and assurance regime (Type C carries the role attribute assist or adjudicate; Section~\ref{sec:3.5}).
\item $\mathrm{Mech}_\kappa$: the participation and assurance mechanism prescribed by the regime. It is instantiated as a deterministic verifier and admission gate (Type A); a calibrated statistical gate (Type B); authorized human acceptance or adjudication (\texttt{C[adjudicate]}); authorized human production or judgement, with AI material restricted to labelled advice (\texttt{C[assist]}); an established deterministic tool path (Type D); or non-AI execution (Type E).
\item $F_\kappa$: the declared fallback, where applicable.
\item $\mathrm{Obl}_\kappa$: the assurance obligations with their evidence kinds and evidence.
\item $G_{r_\kappa}$: the regime's \textbf{guarantee clause}, defined in Sections~\ref{sec:3.3}--\ref{sec:3.6}.
\end{itemize}
\end{definition}

\textbf{Deployment readiness.} An assurance record has a readiness status
\[
\mathrm{readiness}(\mathsf{AR}_\kappa)\in\{\text{READY},\ \text{NOT READY}\}.
\]
It is READY when every obligation in $\mathrm{Obl}_\kappa$ is satisfied with evidence of an adequate kind, and NOT READY otherwise.

\begin{itemize}
\item The regime $r_\kappa$ states which participation and assurance regime the unit requires.
\item Readiness states whether the obligations needed to operate that regime currently hold. It is a status, not a guarantee.
\item AI participates in a unit only under a READY assurance record in a regime that permits AI participation. Where an AI-produced artefact may become the unit's deliverable, it is admitted only through the mechanism $\mathrm{Mech}_\kappa$ prescribed by that regime.
\item A guarantee clause $G_{r_\kappa}$ is claimed only for a READY record.
\end{itemize}

Figure~\ref{fig:3.1} summarizes the framework: the staged assignment of regimes, the record fields of Definition~\ref{def:6} and their separation from deployment readiness.

\begin{figure}[tbp]
\centering
\includegraphics{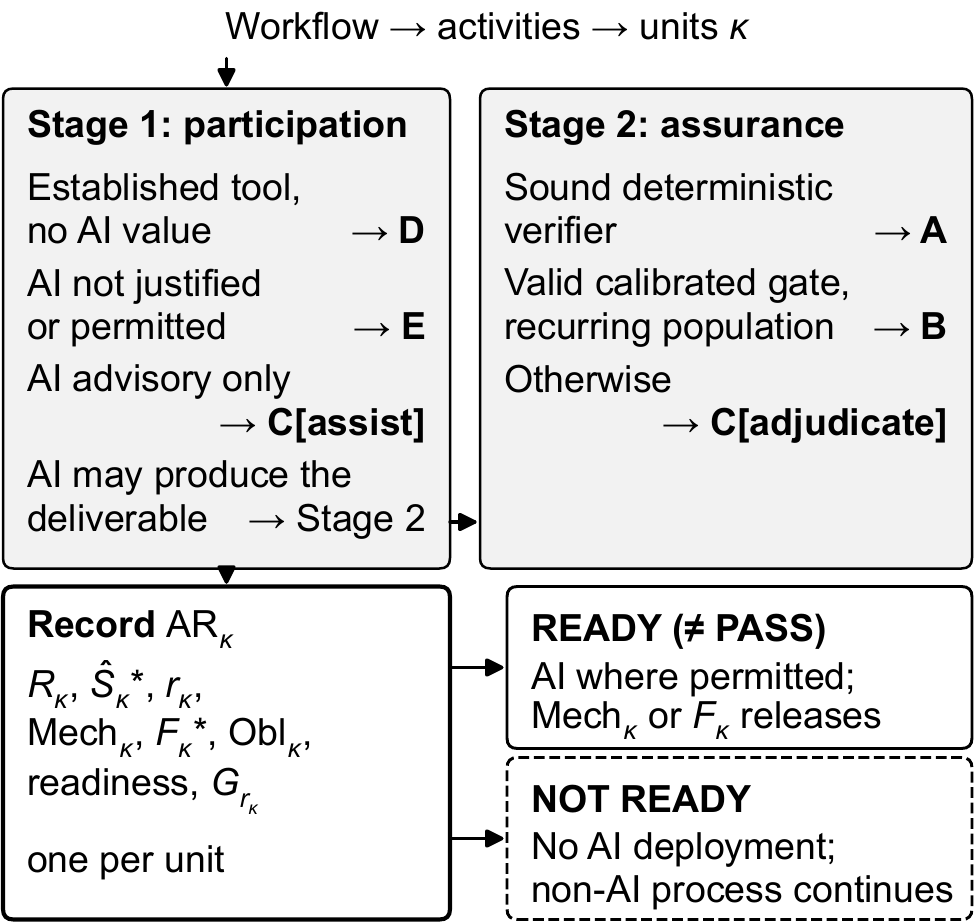}
\caption{Overview of the workflow-unit participation and assurance framework. Regime classification, assurance evidence and deployment readiness are distinct; the regimes carry different guarantee kinds and are not ordered on a scalar assurance scale. Stage 1 (Section~\ref{sec:4.2}) assigns Type D, Type E or C[assist], or passes the unit to Stage 2, which selects the admission rule for units whose deliverable AI may produce. The record box shows the fields of Definition~\ref{def:6} (* where applicable) together with the record's readiness status. Under a READY record, the prescribed mechanism produces, admits or authorizes the deliverable, or a declared fallback produces it; under a NOT READY record (dashed outline), the existing non-AI process produces it (Definition~\ref{def:15}). PASS denotes a favourable study outcome label (Section~\ref{sec:5.6}), which readiness does not imply.}\label{fig:3.1}
\end{figure}

The assurance record is a structured record, not an algebraic object: no refinement or composition operators are defined on it. Its \textbf{definition} (the fields above and the guarantee clause they instantiate), the \textbf{evidence} that its obligations are satisfied (which determines readiness) and the \textbf{mathematical consequences} that follow once the obligations hold are kept apart.

\FloatBarrier
\subsection{Type A — Deterministically Gated Units}\label{sec:3.3}

Some AI-produced artefacts, such as parametric CAD programs whose geometry can be measured, are hard to produce but deterministically checkable. Where the check's soundness does not rest on trust in the generator, AI may participate as an untrusted generator whose artefacts are admitted only if they pass the check.

\begin{definition}[Type A unit and guarantee $G_A$]\label{def:7}

A workflow unit is \textbf{Type A} if:

\begin{itemize}
\item it has an approved formalization $\hat S_\kappa$;
\item a deterministic verifier $V_\kappa:I_\kappa\times O_\kappa\to\{\text{accept},\text{reject}\}$ is implemented and terminates within the unit's time budget;
\item the gate admits an artefact only if $V_\kappa$ accepts it.
\end{itemize}
\end{definition}

The obligations required for $G_A$ are:

\begin{itemize}
\item \textbf{verifier soundness} with respect to the approved $\hat S_\kappa$: $V_\kappa(i,o)=\text{accept}\Rightarrow\hat S_\kappa(i,o)$. This is [Proof] where a proof is available; otherwise [Eng].
\item \textbf{termination} within budget ([Emp]/[Eng]).
\end{itemize}

Further obligations of the record are \textbf{fidelity support} for $\hat S_\kappa$ (Section~\ref{sec:3.8}), which bears on how $G_A$ relates to $R_\kappa$, and a declared fallback.

The guarantee clause is \textbf{deterministic acceptance soundness relative to the approved formalization}: for every $i\in I^{valid}_\kappa$ and every artefact $o$,
\begin{equation}\label{eq:GA}
E(i,o)\ \Rightarrow\ \hat S_\kappa(i,o),
\end{equation}
conditional on verifier soundness.

Clause~\eqref{eq:GA} is a per-instance statement with five boundaries:

\begin{itemize}
\item \textbf{Completeness is not required:} a sound verifier may reject correct artefacts, at a cost in efficiency, not assurance.
\item \textbf{Availability is not guaranteed:} rejections route to the fallback; the generator–verifier acceptance rate should be reported.
\item \textbf{Generator reliability does not enter $G_A$:} conditional on verifier soundness, only artefacts satisfying the approved formalization $\hat S_\kappa$ are admitted.
\item \textbf{Fidelity is not established:} \eqref{eq:GA} concerns $\hat S_\kappa$, not $R_\kappa$ (Section~\ref{sec:3.8}).
\item \textbf{Verifier soundness is itself an obligation} on which the guarantee is conditional.
\end{itemize}

Where soundness cannot be proven, independence between verifier and generator (separate development, diverse tooling, independent review) strengthens the [Eng] evidence for it. Independence is good practice, not a condition of $G_A$: if soundness is proven, \eqref{eq:GA} holds regardless of who produced the verifier.

\textbf{Generation–verification asymmetry.} AI participation can be useful under Type A where an asymmetry holds: generation is difficult or costly relative to a tractable deterministic verification step. This concerns the value of AI participation, not the validity of~\eqref{eq:GA}; it is not part of $G_A$, and Definition~\ref{def:7} does not require it. Where the verifier merely reproduces the unit's own computation, such as evaluating inequalities whose result is the deliverable, an AI generator adds no justified value, and the unit is more accurately treated as Type D (Section~\ref{sec:3.6}) than as Type A.

\emph{Illustration.} In the wing-spar study (Section~\ref{sec:5.3}), an AI-generated CAD program is admitted through a deterministic geometric gate that establishes conformance only within the formalization's declared scope; residual engineering adequacy is a separate human-adjudicated unit (Sections~\ref{sec:3.5} and~\ref{sec:4.3}).

\subsection{Type B — Statistically Gated Units}\label{sec:3.4}

For some workflow units, an approved formalization exists but can be evaluated only on a sample (by expert review or an expensive reference computation), not per instance at deployment. Computable signals (tests, consistency checks, quality metrics, model agreement) can then be combined into a score whose threshold is calibrated against reference evaluations of $\hat S_\kappa$ on a sample from the unit's population. The resulting guarantee is statistical and refers to that population.

Let $D_\kappa$ be a \textbf{declared deployment population} of unit instances, and $G_\kappa$ a \textbf{fixed generator version}. For a gate $g$ with admission event $E_g$, define its \textbf{selective acceptance error} and \textbf{coverage}:
\[
\eta(g)=P_{D_\kappa,G_\kappa}\big(\neg\hat S_\kappa(i,o)\mid E_g\big),\qquad \mathrm{cov}(g)=P_{D_\kappa,G_\kappa}(E_g).
\]
($\eta(g)$ is defined only when $\mathrm{cov}(g)>0$; in~\eqref{eq:GB} below, the event is taken to fail when $\mathrm{cov}(\hat g)=0$.)

\begin{definition}[Calibration procedure]\label{def:8}

A \emph{calibration procedure} $\mathcal K$ maps a calibration sample $\{(i_j,o_j,y_j)\}_{j=1}^m$ to a gate $\hat g=\mathcal K(\cdot)$. Here $i_j$ is drawn from $D_\kappa$, $o_j\sim G_\kappa(i_j)$, and $y_j=\hat S_\kappa(i_j,o_j)$ is obtained from a \textbf{trusted reference labelling procedure}. $\mathcal K$ is \emph{valid for targets} $(\alpha,c,\delta)$ under stated sampling assumptions (for example, independent draws from $D_\kappa$) if, under those assumptions,
\begin{equation}\label{eq:GB}
P_{\text{cal}}\Big(\ \eta(\hat g)\le\alpha\ \ \text{and}\ \ \mathrm{cov}(\hat g)\ge c\ \Big)\ \ge\ 1-\delta,
\end{equation}
where $P_{\text{cal}}$ is taken over the calibration sample. Here:

\begin{itemize}
\item $\alpha\in(0,1)$ is the \textbf{selective-risk target};
\item $c\in(0,1]$ is the \textbf{coverage target};
\item $\delta\in(0,1)$ is the \textbf{calibration confidence parameter}.
\end{itemize}
\end{definition}

\begin{definition}[Type B unit and guarantee $G_B$]\label{def:9}

A workflow unit is \textbf{Type B} if it has an approved formalization $\hat S_\kappa$ and its gate is produced by a calibration procedure valid for stated targets $(\alpha,c,\delta)$, with $c>0$. The obligations are:

\begin{itemize}
\item a declared deployment population $D_\kappa$ ([Eng]);
\item a fixed generator version ([Org]);
\item a trusted reference labelling procedure ([Eng]/[Org]);
\item validity of the calibration procedure under stated assumptions (the procedure's property is [Proof], typically imported from the statistical literature);
\item satisfaction of those sampling assumptions by the calibration and deployment data ([Emp]/[Eng]);
\item monitoring for population shift ([Emp]);
\item reporting of the observed acceptance rate.
\end{itemize}
\end{definition}

The guarantee clause is~\eqref{eq:GB}: \textbf{with confidence at least $1-\delta$ over calibration, the calibrated gate has selective acceptance error at most $\alpha$ and coverage at least $c$ on $D_\kappa$ for generator version $G_\kappa$.}

Several features of this definition are deliberate.

\begin{itemize}
\item \textbf{Selective risk with coverage.} A bound on $P(E\wedge\neg\hat S_\kappa)$ is satisfied trivially by a gate that admits nothing. Clause~\eqref{eq:GB} instead bounds the error \emph{among admitted artefacts}, the quantity relevant to anyone relying on one, and the coverage requirement makes the guarantee non-vacuous and keeps $\eta$ well defined.
\item \textbf{Population- and generator-specific.} The \emph{regime} is determined by the unit and its available assurance mechanism, but a \emph{calibrated Type B guarantee instance} refers to a specific population and generator version, does not transfer, and must be recalibrated after a material change in either.
\item \textbf{Population-level, not per-instance.} Clause~\eqref{eq:GB} does not establish the correctness of any individual artefact.
\item \textbf{One-off units.} If no meaningful recurring population can be declared, the calibration predicate $P_{calib}$ (Section~\ref{sec:4.1}) is false and the unit cannot be Type B; the framework does not introduce subjective probabilities to make a statistical guarantee apply to a one-off decision. Types A, C, D and E do not depend on such a population.
\item \textbf{Agreement is not a guarantee.} Model agreement, tests and metrics may feed the score but are not a statistical guarantee: model errors can be correlated \citep{kim2025correlated,goel2025great}, and agreement says nothing quantitative about correctness unless calibrated against $\hat S_\kappa$. Correlated errors among the combined signals do not by themselves invalidate a valid calibration, but they may reduce the coverage attainable at a given risk target $\alpha$.
\end{itemize}

\textbf{Acceptable calibration methods.} The framework prescribes no particular procedure. Acceptable classes, provided~\eqref{eq:GB} holds under assumptions valid for the unit, include selective classification with guaranteed risk \citep{geifman2017selective}, finite-sample binomial confidence procedures applied to a gate fixed in advance \citep{clopper1934binomial}, Learn-then-Test-type risk control \citep{angelopoulos2025ltt} extended with a coverage constraint, and other statistically valid schemes.

\begin{remark*}[illustrative calibration effort]

In the simplest special case, the gate is fixed in advance (not selected using the calibration data), the admitted calibration instances are independent and each violates $\hat S_\kappa$ with probability $\eta$, and none of the $k$ admitted instances violates $\hat S_\kappa$. A one-sided exact binomial bound at confidence $1-\delta$ then gives $\eta\le1-\delta^{1/k}\approx\ln(1/\delta)/k$, so establishing $\eta\le\alpha$ requires about $k\approx\ln(1/\delta)/\alpha$ admitted instances: for $\delta=0.05$, $k=299$ for $\alpha=10^{-2}$ and $k=2995$ for $\alpha=10^{-3}$. The total sample is larger (roughly $k$ divided by the coverage), the coverage clause needs its own confidence bound, and a threshold selected on the same data changes the requirement. The example is not a general requirement, but \textbf{rigorous Type B assurance requires a reusable operational population and sufficient labelled calibration evidence, and more stringent risk targets generally require larger calibration sets.}
\end{remark*}

\emph{Illustration (hypothetical).} AI-drafted summaries of standardized inspection records, recurring across many parts, are checked against a formalized list of mandatory fields and value ranges, some confirmable only by an engineer. Engineer review of a calibration sample provides reference labels; the calibrated gate admits summaries scoring above its threshold and routes the rest to review. No Type B instance is claimed for the wing-spar study.

\subsection{Type C — Human Judgement and Adjudication}\label{sec:3.5}

For many workflow units, the deliverable or its acceptance cannot be delegated to computation: the requirement may not be formalizable or may admit legitimate expert disagreement, or regulation and professional responsibility may require a named person to decide. AI may still participate, by supplying advisory material or by drafting the artefact; Type C covers both arrangements through one regime with a role attribute.

\begin{definition}[Type C unit and guarantee $G_C$]\label{def:10}

A workflow unit is \textbf{Type C} if its deliverable or decision is released only through a recorded act of an identified, authorized person. The role attribute is one of:

\begin{itemize}
\item \textbf{\texttt{C[assist]}:} AI supplies advisory material (analysis, options, summaries, flagged concerns), and the authorized person produces the deliverable or decision. AI material is retained only as labelled advisory material and is never itself admitted as the deliverable.
\item \textbf{\texttt{C[adjudicate]}:} AI produces a draft artefact, and the draft is released as the deliverable only upon the authorized adjudicator's explicit acceptance, possibly after revision.
\end{itemize}
\end{definition}

The obligations are:

\begin{itemize}
\item the identity and authority of the authorized person ([Org]);
\item a recorded act: acceptance of the AI-produced artefact (\texttt{C[adjudicate]}), or production of the deliverable or decision (\texttt{C[assist]}) ([Org]);
\item recorded AI provenance of every contributing artefact or advisory item ([Org]);
\item a recorded rationale and the evidence consulted ([Org]/[Eng]).
\end{itemize}

These obligations concern readiness, not classification: a unit requiring human judgement or adjudication is Type C even before an authorized process exists, but its record is then NOT READY, and no AI-produced artefact or advisory material is used for it.

The guarantee clause is \textbf{authorized human release}: under \texttt{C[adjudicate]}, no AI-produced artefact of the unit is released as its deliverable without the recorded acceptance of an authorized adjudicator; under \texttt{C[assist]}, the deliverable or decision is produced and recorded by an authorized person, with AI material present only as labelled advice. In both cases, all AI contributions are identifiable.

$G_C$ is procedural. It makes no claim about the probability that the released deliverable is correct, or how that probability compares with unaided adjudication. AI assistance may improve decisions or degrade them, for example through automation bias \citep{parasuraman2010complacency} or anchoring on AI recommendations \citep{bansal2021whole}; which occurs is an empirical question outside the formal guarantee. Mitigations such as independent-first judgement can be recorded as [Eng] obligations, but $G_C$ does not depend on them. Human authority does not imply human correctness; the framework guarantees the former only.

\subsection{Types D and E — Retained Deterministic Paths and Non-Participation}\label{sec:3.6}

Types D and E are participation decisions in which AI-produced artefacts do not contribute to the deliverable.

\begin{definition}[Type D unit and guarantee $G_D$]\label{def:11}

A workflow unit is \textbf{Type D} if two conditions hold:

\begin{itemize}
\item an \textbf{established deterministic tool path} already produces its deliverable, with stated validation or assurance evidence and a declared validity domain;
\item AI participation would add no justified value beyond that path.
\end{itemize}
\end{definition}

Type D is thus a decision to retain an established deterministic path, whether or not AI could assist the unit; it is not a form of verified AI.

The obligations are:

\begin{itemize}
\item the stated validation evidence for the tool path ([Eng]; plus [Org] where a formal qualification record exists);
\item operation within the declared validity domain ([Eng]; the domain check may itself be a separate Type A or D unit).
\end{itemize}

The guarantee clause is \textbf{retention of the established path}: the deliverable is produced by the established tool path on inputs within its validity domain, and no AI-produced artefact substitutes for its output.

Four distinct things are easily conflated:

\begin{enumerate}
\item \textbf{Deterministic solver correctness:} the tool computes what it is specified to compute (e.g., a finite-element solver solves the discretized equations to tolerance).
\item \textbf{Validation evidence:} benchmarks, comparisons with analytical or experimental results, and operational history supporting (1).
\item \textbf{Formal tool qualification:} a recorded qualification under an applicable standard (e.g., RTCA DO-330/EUROCAE ED-215), existing only where actually carried out.
\item \textbf{Engineering model adequacy:} whether discretization, idealization, boundary conditions and analysis type suit the engineering question.
\end{enumerate}

$G_D$ retains the established tool's deterministic correctness, supported by the stated validation evidence and, where available, formal tool qualification. It does \textbf{not} establish engineering-model adequacy; mesh adequacy, idealization and boundary conditions are workflow units in their own right, with their own records.

\begin{definition}[Type E unit and guarantee $G_E$]\label{def:12}

A workflow unit is \textbf{Type E} if AI does not participate in it and the retained-tool rationale of Type D does not apply. Typically, this means the unit is performed by people. The guarantee clause is \textbf{non-delegation}: no AI-produced artefact contributes to the unit's deliverable. The obligation is an auditable record that no AI artefact contributed ([Org]).
\end{definition}

Type E is explicit, justified non-participation, defined narrowly: where AI supplies useful advisory material but a human produces the deliverable, the unit is \texttt{C[assist]}, not Type E.

\subsection{Relations Among the Participation and Assurance Regimes}\label{sec:3.7}

The five guarantee clauses are of different kinds:

\begin{itemize}
\item $G_A$ is a per-instance logical implication;
\item $G_B$ is a population-level statistical statement with a coverage requirement;
\item $G_C$ is a procedural property of release;
\item $G_D$ and $G_E$ are participation boundaries.
\end{itemize}

They are not points on a single scale, and \textbf{no total ordering among the regimes is claimed}.

One comparison is meaningful. For units with the \textbf{same} approved formalization $\hat S$, a Type A gate with a sound verifier admits only artefacts satisfying $\hat S$. Its selective acceptance error is therefore zero on every population within $I^{valid}$ on which it admits anything:
\[
P(\neg\hat S\mid E)=0.
\]
Hence it satisfies the risk part of $G_B$ for every $\alpha$, but not necessarily the coverage part, since a sound verifier may reject most artefacts. The comparison does not extend across different formalizations (a Type A guarantee over a narrow $\hat S$ is not stronger than a Type B guarantee over a broad one), nor to Type C, whose guarantee concerns authority and traceability rather than error rates. Types D and E are boundaries of AI participation, not levels of assurance.

\subsection{Formalization Fidelity}\label{sec:3.8}

The guarantees of Types A and B are stated relative to an approved formalization. Their bearing on the real requirement depends on two chains that meet only at admission:
\begin{gather*}
R_\kappa\ \xrightarrow{\ \text{formalized as}\ }\ \hat S_\kappa\ \xrightarrow{\ \text{implemented by}\ }\ V_\kappa,\\
i\ \xrightarrow{\ G_\kappa\ }\ o\ \xrightarrow{\ \text{evaluated by }V_\kappa\ }\ E\ \text{(admit) or}\ \neg E\ \text{(reject)}.
\end{gather*}
The \textbf{requirement} $R_\kappa$ is interpreted as $\hat S_\kappa$; $\hat S_\kappa$ is implemented as a verifier $V_\kappa$ (Type A), or through the reference labelling against which a gate is calibrated (Type B, where the calibrated gate takes the place of $V_\kappa$). Separately, the generator produces the artefact $o$, which that mechanism evaluates.

Two distinct gaps separate the artefact's admission from the requirement:

\begin{enumerate}
\item \textbf{Requirement $\rightarrow$ formalization.} Is $\hat S_\kappa$ a sufficiently faithful operational interpretation of the declared scope of $R_\kappa$? Regulatory text may be ambiguous, implicit conventions (such as the load level at which a criterion applies) omitted, a threshold mis-transcribed, or the scope narrower than intended.
\item \textbf{Formalization $\rightarrow$ implementation.} Does $V_\kappa$ soundly implement $\hat S_\kappa$? This is the verifier-soundness obligation of Type A. For Type B, it is the correctness of the reference labelling.
\end{enumerate}

The second gap can sometimes be closed by proof. The status of the first depends on the form of the authoritative requirement.

\begin{itemize}
\item \textbf{Informal requirements.} When $R_\kappa$ is regulatory text, design intent or engineering judgement expressed in natural language, the fidelity of $\hat S_\kappa$ cannot, in general, be established by proof in the sense in which verifier soundness can. It is supported by engineering evidence and authorized interpretation. This is the common case in regulated engineering, and it is where formalization requires explicit assurance.
\item \textbf{Requirements that are already formal.} When the authoritative source requirement is itself formal or machine-checkable (for example, a normative data schema or a formally specified interface constraint), the fidelity gap may collapse if $\hat S_\kappa$ is that requirement, or it may admit formal refinement or equivalence evidence, which would be a [Proof] obligation.
\end{itemize}

For informal requirements, the framework distinguishes the \textbf{approval} of a formalization (a governance and engineering-review act, recorded as an [Org] obligation and supported by [Eng] evidence such as clause-by-clause traceability, test cases and independent re-derivation) from its \textbf{actual fidelity}, which may remain imperfect after approval. \textbf{Approval of $\hat S_\kappa$ is not proof that $\hat S_\kappa$ represents $R_\kappa$}, even when the requirement is formal, unless formal equivalence or refinement evidence is actually provided.

\begin{definition}[Formalization task]\label{def:13}

For each workflow unit for which a formalization is produced, the \emph{formalization task} $\tau_{\mathrm{form}}(\kappa)$ is defined by:

\begin{itemize}
\item input: $R_\kappa$ and its relevant context;
\item output: a formalization $\hat S_\kappa$ with declared scope (and, where applicable, a verifier $V_\kappa$);
\item requirement: that $\hat S_\kappa$ adequately represents the declared scope of $R_\kappa$.
\end{itemize}
\end{definition}

The formalization task is itself classified by the procedure of Section~\ref{sec:4}. For natural-language regulatory or engineering requirements, final acceptance of fidelity rests on authorized human judgement (\texttt{C[adjudicate]} where an AI-drafted formalization may be adopted, otherwise \texttt{C[assist]} or Type E), since no computation can decide it. Deterministic sub-checks (well-formedness of an encoding, consistency of units, agreement of a threshold with a cited table, coverage of a finite clause list) may be separated and classified as Type A or D where justified. AI-assisted formalization is thus an explicit, classified task, not silently absorbed into a Type A gate.

A change to an approved formalization, including one decided in a downstream unit, is a change to the output of $\tau_{\mathrm{form}}(\kappa)$. It produces a new version of $\hat S_\kappa$, and every gate that depends on it is re-evaluated.

For the comparison of Section~\ref{sec:3.9}, the residual effect of imperfect fidelity is captured by the \textbf{specification-fidelity residual}
\[
\gamma=P\big(\hat S_\kappa(i,o)\wedge\neg R_\kappa(i,o)\mid E\big),
\]
the probability that an admitted artefact satisfies the formalization but not the requirement. It is defined only where $R_\kappa$ is bivalent, may be difficult or impossible to estimate exactly, and is used only in the optional comparison of Section~\ref{sec:3.9}. Review and approval of $\hat S_\kappa$ do not establish $\gamma=0$; the case $\gamma=0$ is an ideal conditional assumption.

\subsection{Gated-Fallback Comparison with an Incumbent Process}\label{sec:3.9}

The guarantees of Sections~\ref{sec:3.3}--\ref{sec:3.6} make no reference to how the unit was performed before AI was introduced. This section asks when a gated AI-assisted process is at least as likely to produce a deliverable meeting the requirement as the process it supplements, and gives a lower bound with the exact identity behind it. The result is optional, applies only under the conditions stated, and concerns a single workflow unit, not a workflow.

The comparison requires a meaningful operational probability model: a declared population of recurring instances over which the probabilities below are defined. It does not apply automatically to a one-off engineering decision, and the framework does not introduce subjective probabilities to make it apply.

\textbf{Setting.} Consider a workflow unit $\kappa$ (subscripts omitted below) under the following conditions:

\begin{itemize}
\item \textbf{(S1)} The requirement $R$ is bivalent: $R:I\times(O\cup\{\bot\})\to\{0,1\}$, with $R(i,\bot)=0$ (no deliverable does not meet the requirement).
\item \textbf{(S2)} A probability space is defined on which the following are jointly defined: an instance $i\sim D$ from a declared operational population; an AI artefact $o\sim G(i)$; the admission event $E$ of the gate; the output $F(i)$ of the actual fallback; and the output $M(i)$ of an \textbf{incumbent process} $M$, meaning the manual or tool-based process that would perform the unit without AI. $M(i)$ is a counterfactual: it is the output the incumbent would have produced on instance $i$.
\item \textbf{(S3)} $P(E)>0$.
\item \textbf{(S4)} Rejected instances are routed to the fallback: $O_{rel}=o$ on $E$ and $O_{rel}=F(i)$ on $\neg E$.
\end{itemize}

The actual fallback $F$ and the counterfactual incumbent $M$ are distinct objects. They are linked only by the following assumption.

\begin{assumption*}[fallback-performance equivalence]

\[
P\big(R(i,F(i))\wedge\neg E\big)=P\big(R(i,M(i))\wedge\neg E\big).
\]
When $P(\neg E)>0$, this is equivalent to $P(R(i,F(i))\mid\neg E)=P(R(i,M(i))\mid\neg E)$: on the rejected instances, the actual fallback produces a deliverable meeting the requirement with the same probability as the incumbent would have unaided.
\end{assumption*}

This is exactly the assumption needed to cancel the rejected-region terms, and nothing stronger is assumed. It concerns correctness performance on the rejected region, not causal non-interference. For a human fallback, independent-first review or withholding the rejected artefact can provide supporting evidence ([Org]/[Eng]) but does not prove the absence of anchoring; the framework does not model human–AI causal interaction.

\textbf{Quantities.} Define:
\[
d_{gate}=P\big(R(i,O_{rel})\big),\qquad d_M=P\big(R(i,M(i))\big),
\]
\[
\eta=P\big(\neg\hat S(i,o)\mid E\big),\qquad \gamma=P\big(\hat S(i,o)\wedge\neg R(i,o)\mid E\big),\qquad q_M(E)=P\big(\neg R(i,M(i))\mid E\big).
\]

\begin{itemize}
\item $\eta$ is the gate's selective acceptance error.
\item $\gamma$ is the specification-fidelity residual of Section~\ref{sec:3.8}.
\item $q_M(E)$ is the \textbf{incumbent's error rate on the sub-population of instances the gate admits}. It is not the incumbent's overall error rate.
\item $\omega=P\big(\neg\hat S(i,o)\wedge R(i,o)\mid E\big)$ is the \textbf{over-restriction residual}: the probability that an admitted artefact fails the formalization yet meets the requirement. Since $\neg\hat S\wedge R$ implies $\neg\hat S$, $0\le\omega\le\eta$.
\end{itemize}

\begin{proposition}[Gated-fallback comparison]\label{prop:1}
Under (S1)–(S4) and fallback-performance equivalence,
\begin{equation}\label{eq:prop1_bound}
d_{gate}-d_M\ \ge\ P(E)\,\big[\,q_M(E)-\eta-\gamma\,\big].
\end{equation}
\end{proposition}

\emph{Proof.} See Appendix~\ref{app:A}.

The proof establishes the exact identity
\begin{equation}\label{eq:prop1_exact}
d_{gate}-d_M\;=\;P(E)\,\big[\,q_M(E)-\eta-\gamma+\omega\,\big],
\end{equation}
from which~\eqref{eq:prop1_bound} follows because $\omega\ge0$ and $P(E)>0$.

Any claim that “AI-assisted execution is no worse than the incumbent” thus depends on the gate's error on admitted artefacts ($\eta$), the fidelity of the formalization ($\gamma$) and the incumbent's performance \textbf{on the admitted instances} ($q_M(E)$). The last matters: a gate that admits only easy instances must be compared with the incumbent's error on those instances, which may be much smaller than its overall error.

\begin{corollary}[Type A]\label{cor:1}
Suppose the verifier-soundness obligation of a Type A gate holds and the declared population $D$ lies within $I^{valid}$. Then $\eta=0$, hence $\omega=0$ (since $0\le\omega\le\eta$), and by~\eqref{eq:prop1_exact}
\[
d_{gate}-d_M=P(E)\,\big[q_M(E)-\gamma\big].
\]
Under the additional ideal assumption $\gamma=0$,
\[
d_{gate}-d_M=P(E)\,q_M(E)=P\big(E\wedge\neg R(i,M(i))\big)\ \ge\ 0 .
\]
\end{corollary}

Under the ideal assumption $\gamma=0$, the non-degradation conclusion $d_{gate}\ge d_M$ requires no estimate of the incumbent's error rate or of the generator's reliability; the size of the improvement, $P(E)\,q_M(E)$, still depends on the admission probability and on $q_M(E)$. The conclusion is conditional on perfect fidelity, which approval does not establish; the primary Type A guarantee remains $G_A$.

\begin{corollary}[Type B]\label{cor:2}
For a Type B gate, with $D=D_\kappa$ and $G$ the calibrated generator version, consider the event over calibration (of probability at least $1-\delta$) on which $\eta\le\alpha$ and $P(E)\ge c$. On that event, by~\eqref{eq:prop1_bound},
\[
d_{gate}-d_M\ \ge\ P(E)\,\big[q_M(E)-\alpha-\gamma\big].
\]
Consequently, if $\alpha+\gamma\le q_M(E)$, then $d_{gate}\ge d_M$ and moreover
\[
d_{gate}-d_M\ \ge\ c\,\big[q_M(E)-\alpha-\gamma\big]\ \ge\ 0 .
\]
\end{corollary}

Here $q_M(E)$ refers to the sub-population admitted by the calibrated gate, and it must be evidenced for that gate if the comparison is to be claimed.

\begin{corollary}[Types C, D, E]\label{cor:3}
The framework provides no comparison of this form for Types C, D and E.
\end{corollary}

For Types D and E, and for \texttt{C[assist]}, no AI artefact is admitted, so $P(E)=0$ and (S3) fails; for \texttt{C[adjudicate]}, (S1) often fails, and otherwise no record clause bounds $\eta$ for a human gate and fallback-performance equivalence cannot be presumed, because the adjudicator has seen the AI material.

\section{Staged Classification and Readiness Procedure}\label{sec:4}

The framework's staged classification and readiness procedure assigns a regime to each workflow unit and records its readiness, in four stages: decomposition and specification (Stage 0); participation, whether and how AI contributes (Stage 1); assurance, the admission rule where AI may produce the deliverable (Stage 2); and record and readiness (Stage 3).

\subsection{Participation and Assurance Predicates}\label{sec:4.1}

\begin{definition}[Predicates]\label{def:14}

For a workflow unit $\kappa$:

\emph{Participation predicates (Stage 1):}

\begin{itemize}
\item \textbf{$P_{assist}(\kappa)$} — AI can materially contribute an artefact or information to the unit, and such contribution is permitted.
\item \textbf{$P_{incumbent}(\kappa)$} — an established deterministic tool path already produces the unit's deliverable, with stated validation or assurance evidence and a declared validity domain.
\item \textbf{$P_{value}(\kappa)$} — AI participation would provide meaningful, justified value beyond the incumbent path (capability, efficiency or cost), documented with its reasons.
\item \textbf{$P_{deleg}(\kappa)$} — it is technically meaningful and organizationally permitted for an AI-produced artefact to become the unit's deliverable after admission.
\end{itemize}

\emph{Assurance predicates (Stage 2):}

\begin{itemize}
\item \textbf{$P_{formal}(\kappa)$} — an approved formalization $\hat S_\kappa$ exists for a declared scope, and its fidelity obligation is satisfied with stated evidence (Section~\ref{sec:3.8}).
\item \textbf{$P_{sound}(\kappa)$} — a deterministic verifier for $\hat S_\kappa$ is implemented and terminates within budget, and its soundness obligation is satisfied with stated evidence (Definition~\ref{def:7}). Independence of the verifier from the generator strengthens [Eng] evidence of soundness but is not required.
\item \textbf{$P_{calib}(\kappa)$} — a meaningful recurring deployment population is declared, and for it and a fixed generator version, a gate has been produced by a calibration procedure valid for stated targets $(\alpha,c,\delta)$ with $c>0$, using trusted reference labels (Definitions~\ref{def:8}--\ref{def:9}).
\end{itemize}
\end{definition}

\needspace{5\baselineskip}

\textbf{Kinds of predicate.}

\begin{itemize}
\item \textbf{Participation predicates} are engineering and organizational judgements. $P_{value}$ depends on available AI capabilities, is recorded with its justification, and is evaluated only when $P_{incumbent}$ holds.
\item \textbf{Assurance predicates} concern the unit's approved formalization and available assurance mechanisms; the regime they select does not depend on generator capability, and only a Type B \emph{guarantee instance} refers to a generator version.
\item \textbf{$P_{sound}$ and $P_{calib}$} are \textbf{mechanism-availability} predicates: true only when the mechanism actually exists and its defining property (soundness, or calibration validity) is supported, not when it merely could exist. They determine classification; the record's remaining obligations (shift monitoring, acceptance-rate reporting, fallback declaration, provenance, human authority) are assessed under readiness (Section~\ref{sec:4.2}).
\end{itemize}

\needspace{36\baselineskip}
\subsection{Staged Classification Procedure}\label{sec:4.2}

\par\noindent\refstepcounter{procedure}\label{proc:1}\textbf{Procedure~\theprocedure\ (Classification).}\par\nopagebreak\vspace{-\parskip}
\begin{Verbatim}[commandchars=\\\{\},formatcom=\procfont,frame=lines,framesep=2.5mm,rulecolor=\color{black!45}]
Stage 0 — Decompose and specify
  Decompose each activity into workflow units with homogeneous
  participation and assurance characteristics (Section 4.3).
  For each workflow unit κ:
    record R_κ and I_κ^valid;
    where applicable, produce and approve Ŝ_κ with declared scope and version;
    classify τ_form(κ) by this procedure (Section 3.8).

For each workflow unit κ:
  Stage 1 — Participation
    if P_incumbent(κ) and not P_value(κ):   r_κ ← D
    else if not P_assist(κ):                r_κ ← E
    else if not P_deleg(κ):                 r_κ ← C[assist]
    else:                                   AI may produce the deliverable → Stage 2

  Stage 2 — Assurance (AI may produce the deliverable)
    if P_formal(κ) and P_sound(κ):          r_κ ← A
    else if P_formal(κ) and P_calib(κ):     r_κ ← B
    else:                                   r_κ ← C[adjudicate]
    record any other regime attainable in principle as target(κ)

  Stage 3 — Record and readiness
    record AR_κ: R_κ, Ŝ_κ, r_κ, mechanism, fallback where applicable, obligations, evidence
    readiness(AR_κ) ← READY     if every obligation in Obl_κ is satisfied
                                 with adequate evidence
                      NOT READY  otherwise
    AI participates in κ only if
      r_κ ∈ \{A, B, C[assist], C[adjudicate]\} and readiness(AR_κ) = READY;
    where an AI-produced artefact may become the deliverable,
      release follows the mechanism prescribed by r_κ
\end{Verbatim}

Stage 1 exits at the first applicable regime: a unit already produced by an established deterministic tool path to which AI would add no justified value is Type D whether or not AI could assist it, and Type E applies when AI does not participate and that rationale does not apply. For units reaching Stage 2, where AI may produce the deliverable, the procedure selects the admission rule for which a mechanism is available: a deterministic verifier (A), a calibrated statistical gate (B) or authorized human adjudication (\texttt{C[adjudicate]}).

\textbf{Properties by construction.} \textbf{Every workflow unit receives exactly one regime.} Stage 1 either terminates at the first applicable branch in Type D, Type E or \texttt{C[assist]}, or passes the unit to Stage 2. Stage 2 then terminates in Type A, Type B or, when neither admission mechanism is available, \texttt{C[adjudicate]}. Hence \textbf{every unit receives some regime}, and because the procedure exits at the first applicable branch, these assignments are mutually exclusive within the decision procedure. Classification does not depend on readiness, so these properties hold whether or not any record is READY; they are consequences of the procedure's structure, not substantive results about engineering work.

Classification and readiness draw on partly the same evidence but remain distinct statuses. Definitions~\ref{def:7}--\ref{def:12} state what each regime presupposes and guarantees, and Procedure~\ref{proc:1} assigns each unit to exactly one regime. Its predicates $P_{formal}$, $P_{sound}$ and $P_{calib}$ already require the supporting evidence for fidelity, soundness and calibration validity, so that evidence bears on classification as well as on readiness.

\textbf{Classification, target regime and deployment readiness.} Three statuses are recorded for each unit and kept distinct:

\begin{enumerate}
\item \textbf{Regime classification} $r_\kappa$: which participation and assurance regime the unit requires under the staged decision tree. Types A and B are assigned only when their verification or calibration mechanism is actually available; otherwise Stage 2 falls through to \texttt{C[adjudicate]}. The procedure therefore never classifies a unit into A or B on the strength of a mechanism that does not exist.
\item \textbf{Target regime:} a regime attainable in principle but not yet available, for instance a verifier that could be written but has not been implemented and supported by evidence. It records an intended improvement and carries no guarantee.
\item \textbf{Deployment readiness} $\mathrm{readiness}(\mathsf{AR}_\kappa)\in\{\text{READY},\text{NOT READY}\}$: whether every obligation of the assigned assurance record is currently satisfied with adequate evidence (Definition~\ref{def:6}).
\end{enumerate}

A unit is not deployed with AI unless its record is READY. A \texttt{C[adjudicate]} or \texttt{C[assist]} unit without an established authorized judgement or adjudication process remains Type C, is NOT READY for AI deployment, continues under its existing non-AI process, and has no guarantee clause claimed. Readiness also applies to Types D and E: a Type D record whose tool is operated outside its validity domain is NOT READY, and $G_D$ is not claimed. No additional regime is introduced for unready records. This policy concerns the honesty of assurance claims, not a presumption that more human involvement improves outcomes (Section~\ref{sec:3.5}).

\subsection{Composite Activities and Coverage}\label{sec:4.3}

An activity whose requirement has formalizable and non-formalizable parts is decomposed into workflow units, each classified separately: a report, for example, into a unit whose numerical content is checked deterministically against upstream sources and a unit whose completeness and clarity are adjudicated by an engineer.

\begin{definition}[Composite activity and coverage obligation]\label{def:15}

Let an activity $\tau$ be decomposed into workflow units $\kappa_1,\dots,\kappa_k$. Its \emph{composite record} is the tuple $(\mathsf{AR}_{\kappa_1},\dots,\mathsf{AR}_{\kappa_k})$, together with:

\begin{itemize}
\item a \textbf{release rule}: the activity's deliverable is released only if, for every workflow unit, either its assurance record is READY and the prescribed participation and assurance mechanism has produced, admitted or authorized the unit's deliverable, as applicable, or a declared fallback has produced it; or its record is NOT READY and the unit's existing non-AI process has produced it. Under the prescribed mechanism, for Types A and B the gate admits the AI-produced deliverable; for \texttt{C[adjudicate]} the authorized human accepts the AI-produced artefact for release, possibly after revision; for \texttt{C[assist]} the authorized human produces the deliverable or decision, with AI material retained as labelled advice; for Type D the established deterministic path produces the deliverable; and for Type E non-AI execution produces it;
\item a \textbf{coverage obligation}: evidence that the decomposition covers the declared scope of the activity's requirement. Conceptually, with $J_{\mathrm{formal}}=\{\,j:\ \mathsf{AR}_{\kappa_j}\text{ contains an approved }\hat S_{\kappa_j}\,\}$,
\[
R_\tau\ \Leftarrow\ \bigwedge_{j\in J_{\mathrm{formal}}}\hat S_{\kappa_j}\ \wedge\ R_{res},
\]
where $R_{res}$ denotes the residual part of the declared scope of $R_\tau$ that lies within the declared scope of no $\hat S_{\kappa_j}$ and is assigned to the units without an approved formalization (typically human judgement or adjudication); an empty conjunction or an empty residual is read as true.
\end{itemize}
\end{definition}

The coverage obligation is itself a fidelity claim about the decomposition. It is supported by engineering evidence and recorded approval ([Eng]/[Org]), typically through \texttt{C[adjudicate]}, and is not in general mathematically provable. Review narrowing illustrates its importance: when a reviewer is directed to the non-formalized residual only, any part of the requirement that falls neither within some $\hat S_{\kappa_j}$ ($j\in J_{\mathrm{formal}}$) nor within $R_{res}$ is checked by no one.

The composite record carries no scalar combined guarantee. The guarantees of its units are of different kinds (Section~\ref{sec:3.7}) and remain unit-level statements; the probability that all units are correct is not, in general, the minimum of their probabilities of correctness.

\subsection{Reclassification and the Workflow Boundary}\label{sec:4.4}

\textbf{Reclassification and versioning.} A unit's regime may change when its requirement or approved formalization changes, a verifier becomes available with supported soundness, calibration becomes feasible for a declared population, AI capabilities change (altering $P_{value}$ or $P_{deleg}$), organizational permissions change, or the population shifts. Assurance records carry version identifiers, and a regime assignment and readiness status refer to a specific record version. A change may alter readiness without changing the regime: a lapsed human authorization or an unmonitored population shift makes a record NOT READY, and a Type B guarantee instance must be recalibrated after a material change in generator version or population.

\textbf{Workflow boundary.} The framework establishes \textbf{local assurance records for individual workflow units}, each conditional on the unit's inputs lying in the declared valid-input domain. It does not derive:

\begin{itemize}
\item workflow-level dependability;
\item non-degradation at the level of the workflow;
\item independence between units;
\item bounds on error propagation;
\item the safety of the engineered system.
\end{itemize}

Local assurance does not imply workflow-level correctness. A record's guarantee concerns the unit's deliverable \emph{for the inputs received} and is claimed only for inputs in $I^{valid}_\kappa$: a deterministic compliance check fed wrong upstream values may execute correctly on the values it receives, but where those values violate the declared upstream assumptions or lie outside $I^{valid}_\kappa$, the record's guarantee is not claimed, and in no case does the check establish anything about the physical structure. Conversely, a downstream gate may detect upstream errors. Degradation of an individual unit is therefore neither necessary nor sufficient for degradation of the workflow.

Each record states its input assumptions. This makes it possible, in future work, to check whether the guarantees of upstream units support the input assumptions of downstream ones. A composition theory of this kind, including iterative loops in which a downstream decision changes an upstream formalization, is outside the scope of this paper.

\section{Case Study: AI-Assisted Structural Analysis of a Wing-Spar Segment}\label{sec:5}

\textbf{Case-study conduct.} This section reports the controlled execution of the wing-spar workflow; the regime classifications describe AI participation within that executed workflow. A retrospective audit of the original case-study record, which preceded the controlled execution, is reported separately in Supplement S7, and full provenance in Supplement S1.

\subsection{Case-study objective and workflow}\label{sec:5.1}

The case study asks whether the framework can be instantiated and applied to a mixed engineering workflow in which AI participates selectively and its outputs are controlled before they propagate downstream. A structural-analysis workflow suits this question because its workflow units differ in kind:

\begin{itemize}
\item generating a parametric CAD program is open-ended generative work whose product can nevertheless be measured;
\item meshing, finite-element solution and result extraction are already performed by established deterministic tools;
\item the adequacy of idealizations and the treatment of unresolved results remain matters of engineering judgement.
\end{itemize}

The engineering object is a tapered, stiffened, cantilevered aluminium wing-spar segment with lightening holes, idealized as a mid-surface shell model. It is evaluated against project-defined stress and buckling criteria under positive and negative tip loading. Tip deflection and a comparison between linear and geometrically nonlinear analyses are reported and used diagnostically. In the implemented workflow (Figure~\ref{fig:5.1}), AI-generated CAD enters only through a deterministic gate with a declared fallback, downstream computation uses retained deterministic tools, and engineering judgement is applied explicitly.

The participation and assurance records, requirements, acceptance criteria, tool versions and acceptance rules were fixed before generation began. The workflow was then executed once, end to end.

The study is an author-led research case study, not a certification or independent assurance exercise; its scope limits are stated in Section~\ref{sec:6.2}.

\begin{figure}[tbp]
\centering
\includegraphics{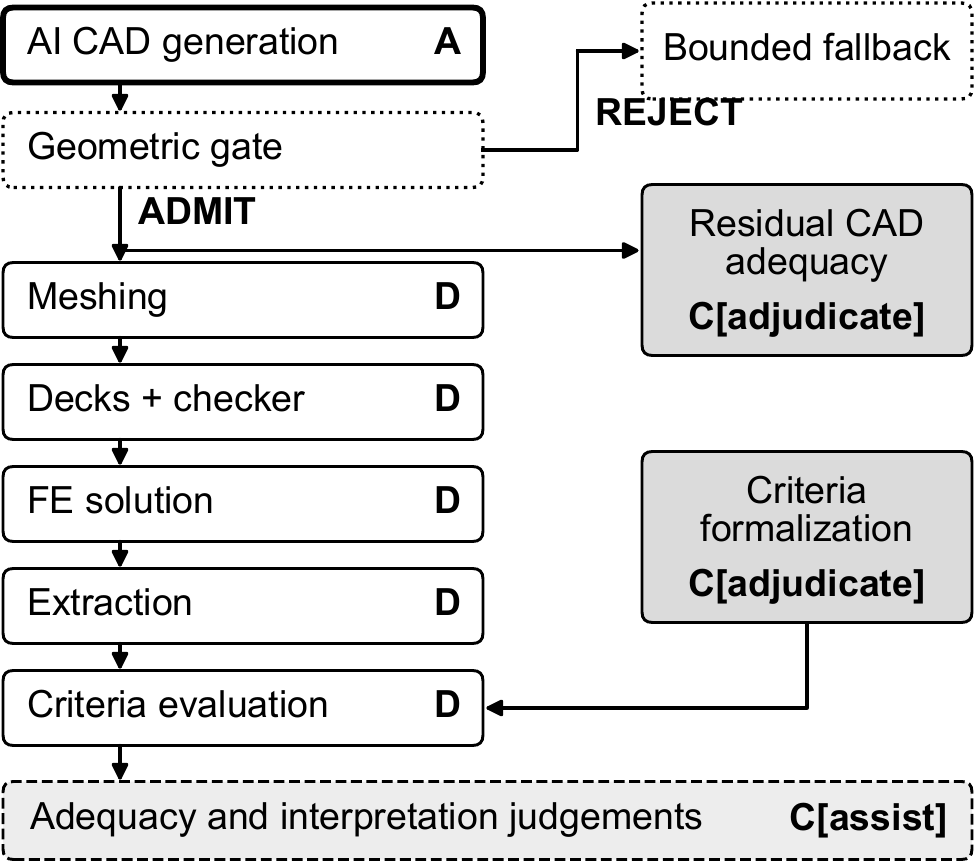}
\caption{Executed workflow and regime allocation. AI-generated CAD (Type A) enters only through a deterministic gate (ADMIT/REJECT) with a predeclared bounded fallback; meshing, deck handling, solution, extraction and evaluation use deterministic tools (Type D). Residual CAD adequacy and criteria formalization are C[adjudicate]; adequacy and interpretation judgements are C[assist]. Representative units are shown; the executed set comprises 17 units (A 1, D 6, C[adjudicate] 4, C[assist] 6; no B or E), all READY. Dotted outlines mark mechanisms, not workflow units. The predeclared fallback allowed at most three attempts through a bounded repair loop, then engineer-authored CAD through the same gate, then a recorded halt; it was not invoked.}\label{fig:5.1}
\end{figure}

\FloatBarrier
\subsection{Instantiating the classification}\label{sec:5.2}

Procedure~\ref{proc:1} decomposes the workflow, which as a whole receives no regime, into a controlled executed set of 17 workflow units. Each unit has a participation and assurance record (regime, mechanism, fallback where applicable, obligations and evidence) recorded before execution, and all 17 assurance records were READY before the run. In this executed set, one unit is Type A, six are Type D, four are C[adjudicate] and six are C[assist]. These counts refer only to the 17-unit executed set, not to the retrospective audit of the original case-study record reported in Supplement S7. Table~\ref{tab:5.1} lists representative workflow units, which together cover every regime instantiated.

\begin{table}[htbp]
\caption{Representative workflow units of the wing-spar workflow and their regimes.}\label{tab:5.1}
\renewcommand{\arraystretch}{1.18}
\begin{tabular}{@{}>{\raggedright\arraybackslash}p{\dimexpr 0.2601\linewidth-1\tabcolsep\relax}>{\raggedright\arraybackslash}p{\dimexpr 0.1958\linewidth-2\tabcolsep\relax}>{\raggedright\arraybackslash}p{\dimexpr 0.1655\linewidth-2\tabcolsep\relax}>{\raggedright\arraybackslash}p{\dimexpr 0.3786\linewidth-1\tabcolsep\relax}@{}}
\toprule
\textbf{Workflow unit} & \textbf{AI role in this execution} & \textbf{Regime} & \textbf{Acceptance / execution mechanism} \\
\midrule
CAD program generation & AI generated the program & A & Deterministic gate measures the produced geometry against an approved operational formalization; admission only if every clause is met; declared fallback \\
Residual CAD adequacy & AI-generated geometry under review; AI compiled evidence & C[adjudicate] & Study author accepts or rejects the artefact on aspects outside the formalized scope \\
Mesh generation & None & D & Scripted mesher with frozen quality rules \\
Solver input-deck generation & None & D & Deterministic deck writer and separate deterministic deck checker \\
Finite-element solution & None & D & Pinned solver within a declared validity domain \\
Result extraction and numerical checks & None & D & Deterministic parsers, equilibrium checks, convergence metrics \\
Formalization of acceptance criteria & AI-assisted draft of the criteria & C[adjudicate] & Study author approves and versions the criteria before any result exists \\
Evaluation of approved criteria & None & D & Deterministic evaluator \\
Adequacy and interpretation judgements (three units) & Advisory: AI compiled evidence and proposed findings & C[assist] & Study author produces the judgement \\
\bottomrule
\end{tabular}
\par\smallskip
\emph{Note.} The AI-role column states the AI contribution recorded in each unit's frozen review or provenance record; “None” means that no AI-produced material entered the unit's execution.
\end{table}

\textbf{Type A.} CAD program generation is the single Type A unit: the program is admitted only through a deterministic geometric gate against an approved operational formalization (Section~\ref{sec:3.3}). Because that formalization covers only specified geometric properties, residual CAD adequacy (whether the geometry suits the analysis and the design intent) is a separate C[adjudicate] unit. Together, the two units are intended to cover the geometric requirement; that coverage is a fidelity claim supported by the recorded approval, not a proved property (Section~\ref{sec:4.3}).

\textbf{Type D.} Meshing, deck generation, finite-element solution, result extraction, evaluation of the approved criteria and engineering report assembly retain established deterministic paths, whose case-specific validation evidence is summarized in Supplement S4. Evaluation is Type D rather than Type A because the deterministic evaluation is the task itself (Sections~\ref{sec:3.3} and~\ref{sec:3.6}).

\textbf{Type C.} C[adjudicate] covers the geometry specification, residual CAD adequacy, the formalization of the acceptance criteria, whose fidelity was not treated as mechanically decidable (Section~\ref{sec:3.8}), and the case-study engineering report narrative. C[assist] covers load-case and idealization decisions, discretization and analysis adequacy, the interpretation of results, the response to unresolved criteria and the study-level review (Section~\ref{sec:3.5}).

\textbf{No Type B or Type E.} Neither is instantiated; Type B would require a recurring deployment population, which a one-off study lacks (Section~\ref{sec:3.4}).

The allocation is therefore selective: deterministically admitted AI generation is confined to the one unit where deterministic verification is tractable; elsewhere, AI-produced drafts are released only through authorized human acceptance (C[adjudicate]), AI advice informs human judgement (C[assist]), and established deterministic tools are retained (Type D).

\FloatBarrier
\subsection{AI generation with deterministic admission}\label{sec:5.3}

In this Type A unit, the generator does not establish the acceptability of its own output. The candidate program was generated by Claude Opus 5.5 in a fresh context with no conversation history. The generator received the frozen generation prompt verbatim. It had no access to the engineer-adopted reference program, the verifier implementation or any downstream result. This is a procedural separation of generation from acceptance, not independence in a formal assurance sense. The model, provider, prompt and generation context of the generation call were recorded (Supplement S1.4). The execution plan specified a “High” reasoning-effort setting, but the model-invocation interface used for the run did not allow that setting to be enforced. The deterministic gate was unaffected, and the run should not be interpreted as evidence of generator performance at a specified reasoning-effort setting.

\textbf{The deterministic gate.} The gate implements the approved geometric formalization. It executes the candidate program within a fixed time budget under execution-containment monitoring, requires exactly one STEP file declaring metre units, and re-imports that file in a fresh process. The re-imported geometry is checked for:

\begin{itemize}
\item topology and face classification;
\item geometric planes and outlines;
\item hole and stiffener geometry;
\item areas and overall extents.
\end{itemize}

Each check has a declared tolerance; the full clause list is in Supplement S3. The gate admits the artefact only if every clause of the formalization is satisfied, including all measured checks.

\textbf{Evidence for the gate.} The gate's soundness is supported by engineering evidence, not by formal proof. Before the study, a pre-run suite of 32 cases reproduced every expected decision. It included controlled mutations of an engineer-adopted reference program, which the gate had to reject, and was complemented by logic-level tests and an author-led code review.

\textbf{Predeclared bounded fallback.} At most three sequential attempts were allowed through a bounded repair loop. After three rejections, an engineer-authored program would have had to pass the same gate; failing that, the unit would halt and the halt would be recorded.

\textbf{Outcome.} The first candidate executed within its budget, produced a single metre-unit STEP file and passed all 23 geometric and validity checks, with no containment violations detected (Table~\ref{tab:5.2}). It was admitted at the first attempt and used unmodified; the repair loop and the fallback were not invoked. A single admission says nothing about the generator's reliability or its admission rate across instances. The assurance comes from the gate, not from the outcome.

\textbf{Residual scope.} Admission establishes conformance to the approved geometric formalization within tolerance, conditional on verifier soundness. It does not establish every engineering property of the artefact. In the separate C[adjudicate] review, the study author judged the geometric representation adequate for the study and recorded, rather than repaired, defensive fallback branches unnecessary in the pinned environment, since the predeclared protocol allowed modification only through the bounded repair loop.

\subsection{Deterministic downstream execution and reproducibility}\label{sec:5.4}

After admission, no further AI-generated engineering artefacts entered the computational chain. The admitted geometry passed to scripted meshing, deterministic deck generation and checking, a pinned solver (CalculiX 2.23) and deterministic extraction with numerical checks; the 21 decks cover linear static, linear buckling and one geometrically nonlinear analysis (Table~\ref{tab:5.2}).

\begin{table}[htbp]
\caption{Execution and integrity evidence.}\label{tab:5.2}
\renewcommand{\arraystretch}{1.18}
\begin{tabular}{@{}>{\raggedright\arraybackslash}p{\dimexpr 0.1635\linewidth-1\tabcolsep\relax}>{\raggedright\arraybackslash}p{\dimexpr 0.4690\linewidth-2\tabcolsep\relax}>{\raggedright\arraybackslash}p{\dimexpr 0.3675\linewidth-1\tabcolsep\relax}@{}}
\toprule
\textbf{Stage} & \textbf{Control / evidence} & \textbf{Observed outcome} \\
\midrule
CAD gate (Type A) & Bounded execution under containment monitoring; 23 geometric and validity checks against the approved formalization; predeclared fallback & ADMIT at attempt 1; 23/23 checks; no containment violations detected; fallback unused; candidate unmodified\textsuperscript{a} \\
Mesh & Lineage check to gate-admitted geometry; frozen quality rules; geometry-conformity check; repeated meshing & 3 levels of quadratic triangular shell elements (3,218 / 12,152 / 47,046 elements); all pass; same-host canonical hashes identical \\
Solver input decks & Separate deterministic deck checker; regeneration comparison & 21/21 pass; byte-identical regeneration \\
Solver & Pinned binary hash; completion and error/warning record & 21/21 complete; no errors or warnings recorded \\
Numerical integrity & Precision-aware force and moment equilibrium; output-consistency conditions & 12/12 linear static jobs pass; all conditions satisfied \\
Frozen inputs & Hash manifest and solver pin re-verified after the run & 61/61 manifest entries and solver binary unchanged (contemporaneous report; later re-hash: Supplement S1.5) \\
\bottomrule
\end{tabular}
\par\smallskip
\textsuperscript{a} An earlier invocation failed before the candidate executed (operator-side path error) and consumed no generation attempt (Supplement S3.5).
\end{table}

\textbf{Forward-carrying controls.} Three controls carry the admission decision forward:

\begin{itemize}
\item the mesher verifies that its input geometry was admitted by the gate and produced by the admitted program;
\item the solver runs only when every deck has passed a separate deterministic checker, which reconstructs the tip load-introduction constraints and confirms the completeness of loads and supports;
\item the evaluator accepts only approved criteria and material data.
\end{itemize}

The controls are independent of the geometry's source by design, although the single run did not test this.

\textbf{Reproducibility.} A 61-entry hash manifest fixed before generation pinned the specifications, criteria, material data, analysis plans, core tools and solver binary. The solver analyses were executed once and not independently repeated. This evidence establishes repeatability of the implemented process, not physical validity of the results, adequacy of the engineering model or formal tool qualification.

\subsection{When a favourable numerical result is not sufficient for PASS}\label{sec:5.5}

\textbf{The convergence rule.} The stress clauses require a quantity both to be converged under the frozen mesh-convergence rule and to meet its threshold. Convergence is judged between the medium and fine meshes, with tolerances fixed before the run:

\begin{itemize}
\item 1 \% for displacement and buckling factors;
\item 5 \% for the stress metric.
\end{itemize}

Tip displacement and both buckling factors converged well within these tolerances. The stress results did not converge uniformly (Figure~\ref{fig:5.2}).

\begin{figure}[tbp]
\centering
\includegraphics{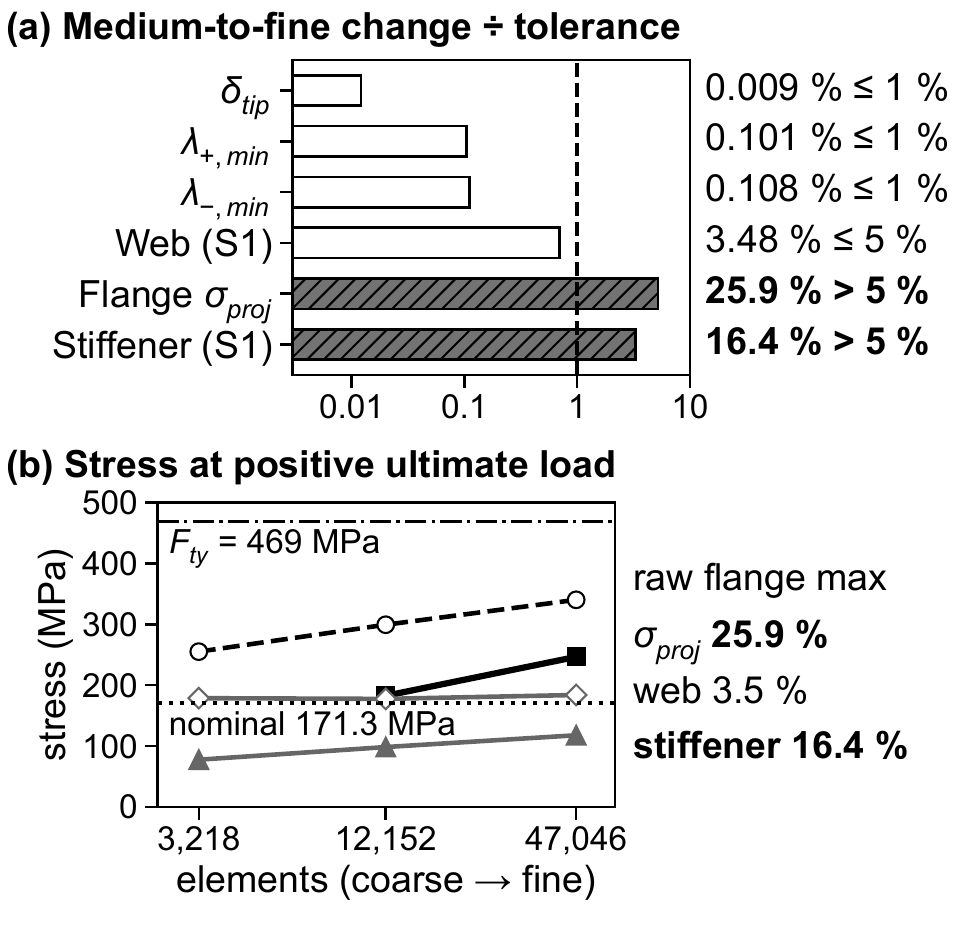}
\caption{Mesh-convergence evidence under the frozen rule. (a) Medium-to-fine relative change of the tip displacement $\delta_{tip}$, the smallest positive buckling factors $\lambda_{+,\min}$ and $\lambda_{-,\min}$, and the web, flange ($\sigma_{proj}$) and stiffener stress metrics (“S1” in the figure denotes the group-maximum stress rule of the frozen engineering plans; Supplement S5.1), divided by the frozen tolerance (1 \% for displacement and buckling factors, 5 \% for the stress metrics; dashed line: tolerance). Hatched bars and bold values exceed the tolerance (NOT\_CONVERGED): the flange project metric $\sigma_{proj}$ (25.9 \%) and the stiffener stress (16.4 \%). (b) Stresses at positive ultimate load versus mesh level (3,218 / 12,152 / 47,046 elements): raw flange root-corner maximum (255.1 / 299.3 / 340.2 MPa), flange project metric $\sigma_{proj}$ (medium 182.6, fine 246.5 MPa), stiffener and web peaks, the nominal root reference (171.3 MPa) and $F_{ty}$ = 469 MPa; bold labels mark NOT\_CONVERGED quantities. The post-hoc interior-flange diagnostic is not plotted and was not a verdict input. Full values: Supplement S5.}\label{fig:5.2}
\end{figure}

\textbf{Three distinct quantities.} The assessment keeps three quantities distinct:

\begin{itemize}
\item the raw finite-element maximum in each component group;
\item the project metric $\sigma_{proj}$ that enters the criteria;
\item a nominal beam-theory reference at the root.
\end{itemize}

\emph{Flange.} At positive ultimate load, the raw flange maximum sits at the clamped root corner and rose with refinement (Figure~\ref{fig:5.2}). The frozen rule anticipated this case. Such a peak is reported separately as a non-converged local clamp peak, not removed as a singularity, and $\sigma_{proj}$ is taken over the flange outside a predeclared root exclusion zone, bounded below by the nominal reference.

In the executed run, however, the maximum outside the exclusion zone ($\sigma_{proj}$) moved to the opposite end of the spar: the tip corner where the load is introduced through the rigid tip constraint, a location the rule had not anticipated. That peak also grew with refinement, from 182.6 MPa to 246.5 MPa (25.9 \%), so $\sigma_{proj}$ was NOT\_CONVERGED.

\FloatBarrier

\emph{Web and stiffener.} The web converged (177.5 MPa to 183.9 MPa, 3.5 \%). The stiffener peak lies at a stiffener–flange junction; junctions are included in the metric by design. It changed from 98.4 MPa to 117.8 MPa (16.4 \%) and was NOT\_CONVERGED. The evidence characterizes the root, tip and junction peaks as mesh-sensitive, unresolved local responses; the study does not establish their mathematical character.

\textbf{Why a low stress ratio did not yield PASS.} For the flange at positive ultimate load, $\sigma_{proj}$ = 246.5 MPa against a yield allowable of $F_{ty}$ = 469 MPa \citepalias{milhdbk5j2003} gives a ratio of 0.526; the other groups and load levels give lower ratios. By magnitude alone this suggests ample margin. But the criterion applies to a converged metric, and the metric was not converged. The evaluator therefore returned NOT\_CONVERGED for all four stress clauses, at limit and ultimate load in both senses, and the study's final disposition was ADJUDICATE.

ADJUDICATE is the study-level disposition used when an evaluated criterion cannot be concluded as PASS under its predeclared rule and requires engineering judgement. Final dispositions (PASS, ADJUDICATE, Not applicable), like the unit outcomes of Section~\ref{sec:5.6}, are study-level outcome labels, distinct from the evaluator's clause verdicts, from deployment readiness and from the regime guarantees. ADJUDICATE is also distinct from the C[adjudicate] participation regime, in which a human accepts or rejects an AI-produced artefact. An ADJUDICATE disposition establishes neither structural failure nor structural safety; it records that the available evidence is insufficient for a conclusion under the predeclared rule.

\textbf{Post-hoc diagnostic (not a verdict input).} Only after the verdicts were recorded did a diagnostic examine the interior flange, 0.20 m $\le$ x $\le$ 2.20 m along the span, which excludes 0.20 m at each end.

\begin{itemize}
\item \textbf{Flange:} the interior stress changed from 179.4 MPa to 179.7 MPa (about 0.2 \%), suggesting that the flange non-convergence is confined to the end regions.
\item \textbf{Stiffener:} the stiffener was not resolved. After excluding stiffener elements within 0.05 m of the flanges, its medium-to-fine change remained about 13 \%.
\end{itemize}

The diagnostic was not part of the predeclared rule, was not passed to the evaluator and changed no verdict. It helps explain the flange outcome; it does not establish compliance.

\textbf{A limitation of the approved formalization.} The predeclared stress rule represented the root-clamp region but not the tip load-introduction region revealed by the run. This is the requirement-to-formalization gap of Section~\ref{sec:3.8}, observed in execution: approval did not ensure that the rule represented the requirement in every region the analysis produced. The outcome was recorded as a limitation, and no criterion was changed. Any refinement of the load introduction, local modelling or stress metric would be a new version of the formalization, to be fixed before a future rerun (Section~\ref{sec:4.4}), not adopted retrospectively. Because the assurance record fixed the metric, the tolerance and the exclusion zone in advance, the gap surfaced as adjudication instead of being tuned away after the result was seen.

\FloatBarrier
\subsection{Final verdicts and engineering adjudication}\label{sec:5.6}

Table~\ref{tab:5.3} lists, for each clause, the deterministic evaluator verdict and the study's final disposition: NOT\_CONVERGED (Q-1, Q-2) becomes ADJUDICATE; PASS (Q-3) remains PASS; the reporting-only verdict REPORTED (Q-4) has no disposition (Not applicable); and SUPPORTED (Q-5) becomes ADJUDICATE, as predeclared.

\begin{table}[htbp]
\caption{Clause verdicts and final dispositions.}\label{tab:5.3}
\renewcommand{\arraystretch}{1.18}
\begin{tabular}{@{}>{\raggedright\arraybackslash}p{\dimexpr 0.4384\linewidth-1\tabcolsep\relax}>{\raggedright\arraybackslash}p{\dimexpr 0.3050\linewidth-2\tabcolsep\relax}>{\raggedright\arraybackslash}p{\dimexpr 0.2566\linewidth-1\tabcolsep\relax}@{}}
\toprule
\textbf{Clause} & \textbf{Evaluator verdict} & \textbf{Final disposition} \\
\midrule
Q-1+ / Q-1$-$ — stress at limit load (positive / negative) & NOT\_CONVERGED & ADJUDICATE \\
Q-2+ / Q-2$-$ — stress at ultimate load (positive / negative) & NOT\_CONVERGED & ADJUDICATE \\
Q-3+ / Q-3$-$ — linear buckling (positive / negative) & PASS & PASS \\
Q-4 — tip deflection (reporting only) & REPORTED & Not applicable \\
Q-5 — linear vs geometrically nonlinear diagnostic & SUPPORTED & ADJUDICATE (predeclared) \\
\bottomrule
\end{tabular}
\par\smallskip
\emph{Note.} Assurance-record readiness: 17/17 READY. Post-run \textbf{unit outcomes} (study-level labels for each workflow unit, distinct from readiness, regime guarantees and the clause verdicts above): PASS 12, ADJUDICATE 3, Not applicable 2, FAIL 0.
\end{table}

\textbf{Buckling: PASS.} The smallest positive buckling factor was $\lambda$ $\approx$ 4.835 under positive load and $\lambda$ $\approx$ 12.086 under negative load, against a project threshold of $\lambda$ $\ge$ 1.5. The evaluator's rule requires a complete run with a usable positive factor, convergence of that factor within 1 \% and the converged factor to meet the threshold. All three held (medium-to-fine change about 10$^{-3}$), so the evaluator returned PASS. The same evaluation that withheld PASS from the stress clauses granted it here, because the evidence the rule required was present.

The C[assist] interpretation identified the first mode in both senses as a global lateral(-torsional) mode of the laterally unrestrained cantilever, so the factors depend on the idealization of tip and lateral restraint. The result is a perfect-geometry linear-bifurcation outcome under the project criteria, not certification-level buckling assurance.

\FloatBarrier

\textbf{Q-5: a favourable cross-check under a predeclared scope limit.} Q-5 compares the linear response at positive ultimate load with a geometrically nonlinear analysis (CalculiX NLGEOM). All relative differences were within the frozen 5 \% threshold, and the evaluator returned SUPPORTED. The flange comparison used the non-converged project metric of Section~\ref{sec:5.5}, so it was treated only as a diagnostic comparison.

The disposition is nevertheless ADJUDICATE. Before any analysis, the study had declared that Q-5 requires adjudication, because the finite-rotation validity of the linearized tip-load relation had not been established. The results support that caution:

\begin{itemize}
\item the tip-deflection ratio at ultimate load ($\delta$/L = 0.0131, with span L = 2.4 m) exceeded the 0.01 escalation threshold;
\item the nonlinear run passed every integrity check but showed a small moment-equilibrium deficit (relative error about 2.01 $\times$ 10$^{-4}$), consistent with the known limitation of the linearized relation. It was treated as diagnostic, not as release-supporting evidence.
\end{itemize}

A favourable numerical cross-check does not override a predeclared scope limitation.

\textbf{The human role.} Human judgement remained explicit: the study author, not AI advice, decided what could be concluded where deterministic evidence stopped short, and no criterion was modified to force acceptance; the guarantee is procedural, not one of correct judgement (Section~\ref{sec:3.5}).

\textbf{READY is not PASS.} All 17 assurance records were READY, meaning that every obligation was satisfied with adequate evidence (Definition~\ref{def:6}), yet readiness did not predetermine the outcome: of the 17 units with READY records, 12 ended PASS, 3 ADJUDICATE and 2 Not applicable (note to Table~\ref{tab:5.3}). Unit outcomes record each unit's post-run result against its own requirement, and they are distinct from clause verdicts. The evaluation unit, for example, recorded PASS because it executed the approved criteria as specified, even though several of the resulting dispositions were ADJUDICATE.

\textbf{What the case study shows.} Within this single, author-led study, the AI-generated CAD artefact entered the analysis only through deterministic admission under a declared fallback, established deterministic tools were retained, and human judgement handled residual and unresolved evidence. The predeclared evidence requirements, not favourable numerical values, controlled the outcome, and the computational process can be inspected and re-executed within the demonstrated scope. These findings concern the implementability and behaviour of the framework in one workflow.

\section{Discussion}\label{sec:6}

\subsection{Transferability beyond the case study}\label{sec:6.1}

The framework's objects (workflow units and their requirements, approved operational formalizations, participation decisions, participation and assurance mechanisms with their fallbacks, evidence obligations and readiness, and the validity domains of established tools) are not specific to structural engineering, and its regimes are defined by the participation and assurance mechanism available for a workflow unit, not by engineering discipline or by a level of AI autonomy. The framework structure is therefore transferable in principle, but each application requires a fresh decomposition, formalization, evidence assessment and validity-domain decision.

Brief illustrations show how the regime semantics carry over; they are not claims of validated deployment.

\begin{itemize}
\item \textbf{Type A} fits workflow units in which AI generates a non-trivial artefact that a sound deterministic verifier checks against an approved formalization. Examples are machine-generated geometry measured against an approved dimensional specification, or a generated configuration checked against a machine-readable interface and rule set.
\item \textbf{Type D.} A deterministically checkable output does not by itself make a unit Type A. Evaluating a computed stress, thermal margin or clearance against an approved limit is Type D when an established deterministic tool path already produces that evaluation (Section~\ref{sec:3.6}).
\item \textbf{Types B, C and E} retain the conditions of Sections~\ref{sec:3.4}--\ref{sec:3.6}: Type B requires a genuinely recurring, calibrated setting; Type C retains authorized human judgement or acceptance; and Type E records justified non-participation.
\end{itemize}

The case for cross-domain applicability therefore rests on the domain-independent structure of the participation and assurance records and the classification procedure. It is an argument from structure, not empirical validation across domains. The case study shows only that the framework can be instantiated and applied in one executed structural-engineering workflow, in which Types A and D, C[adjudicate] and C[assist] were instantiated but neither Type B nor Type E (Section~\ref{sec:5.2}). It does not demonstrate readiness, AI reliability or engineering effectiveness in other domains, nor that any particular workflow will yield a given allocation. Section~\ref{sec:6.4} identifies the evidence that would strengthen the case for transferability.

\subsection{Limitations}\label{sec:6.2}

\subsubsection*{Framework-level limitations}

Predicates such as $P_{value}$, $P_{deleg}$ and $P_{formal}$ record engineering and organizational judgements that the framework does not establish mechanically. Different decompositions may yield different allocations, and coverage of a requirement by a decomposition is itself an assurance obligation, supported by evidence rather than proved (Section~\ref{sec:4.3}).

Approving an operational formalization does not prove its fidelity to the underlying requirement. For natural-language requirements, fidelity rests on engineering evidence and authorized interpretation. Refinement or equivalence evidence is possible only where the authoritative requirement is itself formal (Section~\ref{sec:3.8}).

\needspace{5\baselineskip}

\textbf{Each guarantee carries its own conditions.}

\begin{itemize}
\item \textbf{Type A} is conditional on verifier soundness relative to the approved formalization. Soundness is often supported by testing, mutation testing and review rather than proof; verifier independence may strengthen that evidence but does not establish soundness.
\item \textbf{Type B} is specific to a declared recurring population, a fixed generator version, reference labels and stated calibration assumptions, and is unavailable for one-off units.
\item \textbf{Type C} is procedural: authorized human acceptance or judgement does not ensure correctness, and automation bias, anchoring and weak review remain possible.
\item \textbf{Type D} retains an established path only within its declared validity domain and stated evidence, and does not cover model adequacy.
\end{itemize}

A READY assurance record has its obligations satisfied with adequate evidence, but judging adequacy is itself a judgement. Readiness neither replaces the regime's guarantee nor implies a PASS (Section~\ref{sec:5.6}).

Assurance records are local to workflow units: the framework derives no workflow-level dependability, error-propagation bound or system-safety claim (Section~\ref{sec:4.4}). Proposition~\ref{prop:1} (Section~\ref{sec:3.9}) concerns a single workflow unit under its stated recurring-probability setting and assumptions, including fallback-performance equivalence. It is not a workflow result and does not apply to one-off decisions such as the wing-spar case.

Requirements, formalizations, verifiers, tools, generator versions and populations change, so assurance records must be versioned and re-evaluated. Type B gates must be recalibrated after a material change in population or generator version, and changed verifier or tool evidence may require renewed assurance for Type A and Type D.

\subsubsection*{Case-study limitations}

The study was author-led. Its engineering reviews and adjudications were decisions of the study author, not independent third-party assurance.

The structural model is idealized:

\begin{itemize}
\item a mid-surface shell representation;
\item a linearized rigid tip-load relation without established finite-rotation validity;
\item perfect-geometry linear buckling;
\item project criteria motivated by, but not equivalent to, certification requirements.
\end{itemize}

The results are interpretable only within these modelling choices.

The material allowables did not govern any disposition: the stress-clause dispositions were governed by non-convergence rather than by exceedance of an allowable. The allowables were taken from MIL-HDBK-5J~(\citeyear{milhdbk5j2003}), Table 3.7.6.0(b1), for 7075-T651 plate, A-basis. The governing $F_{ty}$ values are 469 MPa for the 18 mm flange and 462 MPa for the 8 mm web and stiffeners, using the minimum of the listed L and LT values for the corresponding thickness ranges.

\needspace{5\baselineskip}

\textbf{Execution and reproducibility.}

\begin{itemize}
\item OS-level isolation of the AI-generated code execution was not demonstrated.
\item Mesh determinism was shown on the same host only.
\item The solver analyses were not independently repeated, and cross-platform reproducibility was not tested.
\item Reproducibility of the implemented process is not evidence of physical model validity.
\end{itemize}

No comparable manual baseline exists, so no reduction in engineering time, effort or cost is claimed.

Within these boundaries, the case study demonstrates that the framework can be instantiated and applied in an executed engineering workflow with selective AI participation, deterministic admission of the AI-generated artefact, inspectable evidence handling and explicit routing of unresolved evidence to an ADJUDICATE disposition and engineering judgement. It does not establish structural safety, certification or regulatory compliance, independent assurance, general AI reliability, cross-platform reproducibility, workflow-level non-degradation, system-level dependability or productivity improvement.

\subsection{Dependability, certification and accountability}\label{sec:6.3}

Dependability motivates the framework, but the framework addresses it at a specific level: how AI participates in individual workflow units, and what assurance mechanism and evidence justify that participation.

Certification is a different matter. It is a determination made under applicable standards, organizational processes, independence requirements and the authority of the relevant regulator. Participation and assurance records can support assurance and certification activities by making the applicable mechanisms, evidence, fallbacks, readiness and formalization versions explicit and inspectable. The framework does not itself establish certification, regulatory compliance or the safety of the engineered system.

The unit-level perspective complements frameworks that characterize AI applications by the relationship between human and machine authority, such as EASA's Level 1 and Level 2 framework for machine-learning applications. Such levels concern how authority and autonomy are distributed between human and AI. The present framework asks, unit by unit, how AI participates and under what assurance mechanism an artefact or decision may enter the engineering workflow. A single application could be described in both ways, and the two descriptions answer different questions; no mapping or equivalence between EASA levels and the participation and assurance regimes is claimed.

For governance purposes, the same unit-level structure identifies, for each regime, where participation, acceptance or execution decisions are taken and where the associated evidence is held (Table~\ref{tab:6.1}). This is a description of governance, not a theory of legal liability.

\begin{table}[htbp]
\caption{Decision or execution locus and evidence/governance focus by regime.}\label{tab:6.1}
\renewcommand{\arraystretch}{1.18}
\begin{tabular}{@{}>{\raggedright\arraybackslash}p{\dimexpr 0.1641\linewidth-1\tabcolsep\relax}>{\raggedright\arraybackslash}p{\dimexpr 0.3041\linewidth-2\tabcolsep\relax}>{\raggedright\arraybackslash}p{\dimexpr 0.5318\linewidth-1\tabcolsep\relax}@{}}
\toprule
\textbf{Regime} & \textbf{Decision or execution locus} & \textbf{Evidence and governance focus} \\
\midrule
A & Deterministic verifier against an approved formalization & Approval of the formalization; evidence for verifier soundness; governance of admission and fallback; versioning of verifier and formalization \\
B & Calibrated statistical gate & Declared recurring population; fixed generator version; reference labels; calibration assumptions; monitoring and recalibration after material change \\
C[adjudicate] & Authorized human acceptance or rejection of an AI-produced artefact & Identified, authorized adjudicator; recorded acceptance, rationale and AI provenance \\
C[assist] & Authorized human judgement, with AI material as labelled advice & Identified, authorized person producing the judgement or deliverable; recorded rationale and provenance of AI advisory material \\
D & Established deterministic tool path & Stated validation evidence; operation within the declared validity domain \\
E & Non-AI execution & Governance as in the non-AI activity; record that no AI artefact contributed \\
\bottomrule
\end{tabular}
\end{table}

Formalization approval is an accountable governance act, not proof of fidelity (Section~\ref{sec:3.8}). The framework does not determine who holds legal duties, what level of independence a regulator requires, which standards apply, or whether certification is granted; these remain external to it.

\FloatBarrier
\subsection{Future work}\label{sec:6.4}

Four directions follow from Sections~\ref{sec:5} and~\ref{sec:6.2}.

\needspace{5\baselineskip}

\textbf{Stronger assurance evidence.}

\begin{itemize}
\item The link from requirement to approved formalization would benefit from systematic traceability, independent re-derivation and review, and structured change control.
\item Where feasible, Type A verifier soundness could be strengthened by formal proof or proof-producing verification. This would strengthen Type A, not become a requirement for every Type A unit.
\item Execution assurance would benefit from OS-level isolation of AI-generated code, and reproducibility evidence from independent reruns and cross-platform repetition.
\end{itemize}

\needspace{5\baselineskip}

\textbf{Broader empirical evidence.}

\begin{itemize}
\item A genuinely recurring deployment satisfying the population, calibration and monitoring conditions of Section~\ref{sec:3.4} is needed to exercise Type B empirically.
\item An independent second engineering case, with independent decomposition, classification and formalization, would add evidence of transferability without by itself establishing it across domains.
\item Type C units raise empirical questions of automation bias, anchoring, the design of advisory material, reviewer independence and the quality of recorded rationale.
\item Any efficiency study would require a predefined comparative baseline and measurement protocol, and its outcome cannot be presumed.
\end{itemize}

\textbf{Case-specific refinement.} The wing-spar stress rule anticipated the root clamp region but not the tip load-introduction region (Section~\ref{sec:5.5}). Any improved metric, exclusion rule or local model would be a new formalization version, fixed before a new run, not tuned retrospectively to convert the present ADJUDICATE outcome into PASS.

\textbf{Architecture and theory.} Assurance records remain local. A composition theory remains open (Section~\ref{sec:4.4}). It would need to cover:

\begin{itemize}
\item the propagation of rejected outputs and fallback deliverables;
\item iterative workflows in which a downstream decision revises an upstream formalization;
\item dependencies among assurance records.
\end{itemize}

Recurring workflow-unit patterns might support reusable assets: assurance-record templates, formalization and evidence patterns, verifier and tool-path records, and classification support. Their reuse would remain subject to checks of scope, version and validity domain. Any automated classifier would itself be a workflow unit to be classified and assured under the framework.

\section{Conclusion}\label{sec:7}

Engineering workflows contain units of work that call for different forms of AI participation and assurance. This paper has developed a formal framework that pairs an explicit AI-participation decision for each workflow unit with the assurance required to justify it. Each unit's participation and assurance record identifies the requirement, any approved operational formalization, the applicable mechanism, any declared fallback, evidence obligations with explicit evidence kinds and the regime guarantee, and carries a deployment-readiness status. The regimes (deterministic or statistically calibrated admission, authorized human judgement or adjudication, retention of an established deterministic path, and explicit non-participation) carry different kinds of guarantee and admit no scalar ordering. The framework also separates formalization fidelity from verifier soundness, keeps classification, target regime and deployment readiness distinct through a staged procedure, and states conditions under which a gated unit can be compared with an incumbent process on a recurring population.

The executed wing-spar case showed that the framework can be instantiated and applied in an engineering workflow with selective AI participation. The AI-generated CAD artefact entered the analysis only through deterministic admission, established deterministic tools were retained, and human judgement and adjudication handled residual and unresolved evidence. The predeclared evidence requirements, not favourable numerical values, controlled the outcome: favourable stress values did not yield PASS without the required convergence evidence, and READY assurance records did not imply PASS.

These conclusions concern local, unit-level assurance in a single author-led case and are bounded by the limitations stated in Section~\ref{sec:6.2}; validating the framework beyond this case will require independent engineering applications.

\bibliographystyle{unsrtnat}
\bibliography{references_v1}

\appendix
\renewcommand{\thesection}{A}\renewcommand{\theHsection}{appA}
\refstepcounter{section}\label{app:A}
\section*{Appendix A. Proof and derivational details for Proposition~\ref{prop:1}}
\renewcommand{\thesubsection}{A.\arabic{subsection}}\renewcommand{\theHsubsection}{appA.\arabic{subsection}}

Notation, the setting (S1)–(S4), fallback-performance equivalence and the quantities $d_{gate}$, $d_M$, $\eta$, $\gamma$, $q_M(E)$ and $\omega$ are as defined in Section~\ref{sec:3.9}.

\subsection{Proof of Proposition~\ref{prop:1}}\label{app:A.1}

\emph{Proof.} Recall the over-restriction residual $\omega=P\big(\neg\hat S(i,o)\wedge R(i,o)\mid E\big)$ of Section~\ref{sec:3.9}.

\begin{enumerate}
\item By (S4) and the law of total probability,
\[
d_{gate}=P\big(R(i,o)\wedge E\big)+P\big(R(i,F(i))\wedge\neg E\big),\qquad d_M=P\big(R(i,M(i))\wedge E\big)+P\big(R(i,M(i))\wedge\neg E\big).
\]
\item By fallback-performance equivalence the second terms coincide, so
\[
d_{gate}-d_M=P(E)\big[P(R(i,o)\mid E)-P(R(i,M(i))\mid E)\big].
\]
\item The second conditional probability is $1-q_M(E)$.
\item For the first, partition $\neg R(i,o)$ according to $\hat S(i,o)$:
\[
P(\neg R\mid E)=P(\hat S\wedge\neg R\mid E)+P(\neg\hat S\wedge\neg R\mid E)=\gamma+(\eta-\omega),
\]
because $\eta=P(\neg\hat S\wedge R\mid E)+P(\neg\hat S\wedge\neg R\mid E)=\omega+P(\neg\hat S\wedge\neg R\mid E)$. Hence $P(R(i,o)\mid E)=1-\eta-\gamma+\omega$.
\item Substituting gives the exact identity
\[
d_{gate}-d_M\;=\;P(E)\,\big[\,q_M(E)-\eta-\gamma+\omega\,\big].
\]
\item Since $\omega\ge0$ and $P(E)>0$, \eqref{eq:prop1_bound} follows. $\square$
\end{enumerate}

\subsection{Remark on the over-restriction residual}\label{app:A.2}

The over-restriction residual $\omega$ is required for the exact decomposition~\eqref{eq:prop1_exact}. It is not needed for~\eqref{eq:prop1_bound}, or for the sufficient conditions for non-degradation derived from~\eqref{eq:prop1_bound}.

\subsection{Corollary~\ref{cor:1} (Type A): further detail}\label{app:A.3}

Under $\gamma=0$, the gain is strictly positive exactly when the incumbent errs with positive probability on instances the gate admits. Its size is limited by the admission probability, so the practical benefit depends on how often the generator produces artefacts that the gate admits.

\subsection{Corollary~\ref{cor:2} (Type B): further detail}\label{app:A.4}

If $\alpha+\gamma>q_M(E)$, the bracket is negative, and the coverage bound does not yield a lower bound in that direction.

\subsection{Corollary~\ref{cor:3} (Types C, D, E): derivational detail}\label{app:A.5}

\begin{itemize}
\item \textbf{Types D and E, and \texttt{C[assist]}:} no AI artefact is admitted, so $P(E)=0$ and (S3) fails. Under Types D and E the workflow unit is executed without AI; under \texttt{C[assist]} AI material is advisory only.
\item \textbf{\texttt{C[adjudicate]}:} $R$ is often not bivalent, so (S1) fails. Where it is bivalent, \eqref{eq:prop1_bound} and~\eqref{eq:prop1_exact} remain algebraically valid. However, no record clause bounds $\eta$ for a human gate, and fallback-performance equivalence cannot be presumed, because the person who adjudicates, and who may then produce the deliverable, has seen the AI material. Any comparison would rest entirely on empirical human-factors evidence, outside the framework's guarantees.
\end{itemize}

\clearpage
\section*{Supplementary Material}\label{sec:supplement}
\raggedbottom
\setlength{\LTpre}{4pt}\setlength{\LTpost}{6pt}
\setcounter{section}{0}
\renewcommand{\thesection}{S\arabic{section}}\renewcommand{\theHsection}{supS\arabic{section}}
\renewcommand{\thesubsection}{\thesection.\arabic{subsection}}\renewcommand{\theHsubsection}{supS\arabic{section}.\arabic{subsection}}
\textbf{Nature of the case study.} This is an author-led research case study showing how the framework is instantiated and applied in an executed mixed engineering workflow: decomposition into workflow units, unit-by-unit classification of AI participation, operation of the regimes, the points at which engineering judgement remains necessary, and routing of insufficient evidence to an ADJUDICATE disposition rather than PASS. It is not a certification, airworthiness or independent assurance exercise; its scope limits are those stated in Section~\ref{sec:6.2}.

\textbf{Responsibility.} Puxue Tan was the study lead, domain engineer and author reviewer, and retained responsibility for the engineering judgements and approvals required by the executed study. Deterministic steps relied on their defined mechanisms, and their recorded tool outputs are the evidence. Human review was used where the assigned regime required engineering judgement (S6).

Except for the explicitly identified later re-hash of 2 October 2026 (S1.5), the reported study values are taken from the frozen execution records; no engineering analysis was rerun for the Supplement.

\needspace{6\baselineskip}

\section{Executed workflow and study setup}\label{sup:S1}

\subsection{Engineering problem}\label{sup:S1.1}

The object is a tapered, stiffened, cantilevered aluminium wing-spar segment (span L = 2.4 m) with four web lightening holes and three one-sided stiffeners. It is idealized as a mid-surface shell model with an 18 mm flange and an 8 mm web and stiffeners, made of 7075-T651. The load is a tip load in positive and negative senses at limit and ultimate levels, with load factors 1.0, 1.5, $-$0.4 and $-$0.6 relative to the positive limit load.

The analyses are:

\begin{itemize}
\item linear static analysis;
\item two-sense linear buckling;
\item one geometrically nonlinear (NLGEOM) cross-check at positive ultimate load.
\end{itemize}

They are evaluated against project-defined stress and buckling criteria (S5).

\needspace{17\baselineskip}

\subsection{Workflow decomposition and regime allocation}\label{sup:S1.2}

Procedure 1 decomposed the workflow into \textbf{17 workflow units}; the workflow as a whole receives no regime.

\par\smallskip\noindent{\renewcommand{\arraystretch}{1.18}%
\begin{tabular}{@{}>{\raggedright\arraybackslash}p{\dimexpr 0.2141\linewidth-1\tabcolsep\relax}>{\raggedright\arraybackslash}p{\dimexpr 0.6897\linewidth-2\tabcolsep\relax}>{\raggedright\arraybackslash}p{\dimexpr 0.0962\linewidth-1\tabcolsep\relax}@{}}
\toprule
\textbf{Regime} & \textbf{Units} & \textbf{Count} \\
\midrule
Type A & CAD program generation & 1 \\
Type D & Meshing; deck generation and checking; finite-element (FE) solution; extraction and numerical checks; evaluation of approved criteria; engineering report assembly & 6 \\
C[adjudicate] & Geometry specification; residual CAD adequacy; formalization of acceptance criteria; case-study engineering report narrative & 4 \\
C[assist] & Discretization adequacy; load-case and idealization decisions; analysis adequacy; result interpretation; response to unresolved criteria (conditional); study-level review & 6 \\
Type B, Type E & — & 0 \\
\bottomrule
\end{tabular}}\par\medskip

All 17 participation and assurance records were \textbf{READY before execution} (S2). These counts refer only to the controlled executed set; the 22-item retrospective audit in S7 is a separate denominator.

\needspace{6\baselineskip}

\subsection{Architecture of the executed mixed workflow}\label{sup:S1.3}

The executed workflow combined AI, deterministic tools and engineering judgement as follows (Figure 5.1):

\begin{itemize}
\item \textbf{AI was used selectively.} It produced the CAD program (Type A). In C[adjudicate] units, AI-produced drafts were subject to acceptance, revision or rejection by the study author. In C[assist] units, AI supplied labelled advisory material, evidence or proposed findings, and the study author produced the judgement.
\item \textbf{The AI-generated CAD artefact entered only through deterministic Type A assurance:} a geometric gate with a bounded fallback (S3).
\item \textbf{The deterministic computational chain itself ran automatically through established tools} (S4–S5):
\begin{itemize}
\item scripted meshing at three refinement levels;
\item deck generation with an independent deck checker;
\item a pinned FE solver;
\item deterministic extraction with numerical-integrity checks;
\item deterministic evaluation of the approved criteria.
\par
Engineers did not manually repeat or re-perform this deterministic computation.
\end{itemize}
\item \textbf{Separate human-judgement units remained at the predefined adequacy, interpretation and adjudication points} (S2, S6), where the framework assigns C[adjudicate] or C[assist], including:
\begin{itemize}
\item residual CAD adequacy;
\item formalization of the acceptance criteria;
\item load-case and idealization decisions;
\item discretization and analysis adequacy;
\item result interpretation;
\item the response to unresolved criteria.
\end{itemize}
\item \textbf{Unresolved evidence was escalated, not converted into PASS.} Four stress clauses and the Q-5 cross-check were dispositioned ADJUDICATE (S5.5).
\end{itemize}

The executed workflow thus combined an automatically executed deterministic chain with explicit control by the study author at judgement points. This is an initial demonstration of the framework's structure, not an industrial productivity benchmark. No automation percentage, time saving or optimality is claimed.

\needspace{19\baselineskip}

\subsection{AI generation context (Type A unit)}\label{sup:S1.4}

\par\smallskip\noindent{\renewcommand{\arraystretch}{1.18}%
\begin{tabular}{@{}>{\raggedright\arraybackslash}p{\dimexpr 0.1758\linewidth-1\tabcolsep\relax}>{\raggedright\arraybackslash}p{\dimexpr 0.8242\linewidth-1\tabcolsep\relax}@{}}
\toprule
\textbf{Item} & \textbf{Value} \\
\midrule
Provider / interface & Anthropic model-invocation interface \\
Model & Claude Opus 5.5 (\texttt{claude-opus-5-5}); the interface does not echo a runtime model identifier \\
Context & Fresh context, no conversation history; no access to the reference program, the verifier or downstream results \\
Prompt & Frozen generation prompt, sent verbatim (prompt sha256 \texttt{779754cf…21ad6850f0}), with a text-only instruction envelope \\
Reasoning effort & The plan specified “High”; the interface did not allow this to be enforced. \textbf{The run is therefore not evidence of generator performance at a specified reasoning-effort setting.} The deterministic gate was unaffected. \\
Output & One candidate program (sha256 \texttt{1a69e827…ff72fa8f17}), generated in 39.7 s \\
\bottomrule
\end{tabular}}\par\medskip

\needspace{6\baselineskip}

\subsection{Frozen study setup and reproducibility}\label{sup:S1.5}

Before generation, a hash manifest of \textbf{61 entries} pinned the following, together with the CalculiX 2.23 solver binary (sha256 \texttt{2e9c75bb…598bb8ee}); manifest content hash \texttt{f55b65c8…49bb6173e}:

\begin{itemize}
\item the specifications: geometry, geometric formalization, load specification and material data;
\item the acceptance criteria;
\item the analysis plans and metric definitions;
\item the gate, mesher, FE tools and evaluator;
\item the environment locks.
\end{itemize}

The executed study is defined by these frozen versions and hashes.

\textbf{Post-run verification.} The contemporaneous final report records a 61/61 post-run match, but a standalone machine-readable output of that check was not preserved. A later re-hash on 2 October 2026 reproduced the 61/61 result and the solver-binary hash. This later observation is reported separately and does not replace the original execution-time evidence.

The solver analyses were executed once and not independently repeated. The evidence establishes repeatability of the implemented process, not physical validity of the results or formal tool qualification.

\needspace{32\baselineskip}

\section{Participation and assurance allocation (17 workflow units)}\label{sup:S2}

\textbf{Readiness.} All 17 records were READY before the run: every obligation in the record was satisfied with evidence of an adequate kind. Readiness is a status, not a result. The unit outcomes are reported in S5.6.

\textbf{Table S2.1 — Allocation}

The AI-role column states the AI contribution recorded in each unit's frozen review or provenance record; “None” means that no AI-produced material entered the unit's execution.

{\renewcommand{\arraystretch}{1.18}
\begin{longtable}{@{}>{\raggedright\arraybackslash}p{\dimexpr 0.1019\linewidth-1\tabcolsep\relax}>{\raggedright\arraybackslash}p{\dimexpr 0.3814\linewidth-2\tabcolsep\relax}>{\raggedright\arraybackslash}p{\dimexpr 0.3035\linewidth-2\tabcolsep\relax}>{\raggedright\arraybackslash}p{\dimexpr 0.2133\linewidth-1\tabcolsep\relax}@{}}
\toprule
\textbf{\#} & \textbf{Workflow unit} & \textbf{AI role in this execution} & \textbf{Regime} \\
\midrule
\endfirsthead
\toprule
\textbf{\#} & \textbf{Workflow unit} & \textbf{AI role in this execution} & \textbf{Regime} \\
\midrule
\endhead
\bottomrule
\endlastfoot
T1 & Geometry specification & AI-originated baseline values; adopted by the study author & C[adjudicate] \\
T2a & CAD program generation & AI generated the program & A \\
T2b & Residual CAD adequacy & AI-generated geometry under review; AI compiled evidence & C[adjudicate] \\
T3a & Meshing & None & D \\
T3b & Discretization adequacy & Advisory: AI compiled evidence and proposed findings & C[assist] \\
T4a & Deck generation and checking & None & D \\
T4b & Load cases and idealization & Advisory: AI compiled evidence & C[assist] \\
T5a & FE solution & None & D \\
T5b & Analysis adequacy & Advisory: AI compiled evidence and proposed findings & C[assist] \\
T6a & Extraction and numerical checks & None & D \\
T6b & Result interpretation & Advisory: AI compiled evidence and proposed findings & C[assist] \\
T7a & Formalization of acceptance criteria & AI-assisted draft of the criteria & C[adjudicate] \\
T7b & Evaluation of approved criteria & None & D \\
T8a & Engineering report assembly (numbers) & None & D \\
T8b & Case-study engineering report narrative & AI drafted the narrative & C[adjudicate] \\
T9 & Response to unresolved criteria (conditional) & Advisory: AI compiled evidence and proposed findings & C[assist] \\
T10 & Study-level review & Advisory: AI compiled evidence and proposed findings & C[assist] \\
\end{longtable}}

\needspace{32\baselineskip}

\textbf{Table S2.2 — Mechanism and fallback}

{\renewcommand{\arraystretch}{1.18}
\begin{longtable}{@{}>{\raggedright\arraybackslash}p{\dimexpr 0.0906\linewidth-1\tabcolsep\relax}>{\raggedright\arraybackslash}p{\dimexpr 0.4399\linewidth-2\tabcolsep\relax}>{\raggedright\arraybackslash}p{\dimexpr 0.2910\linewidth-2\tabcolsep\relax}>{\raggedright\arraybackslash}p{\dimexpr 0.1785\linewidth-1\tabcolsep\relax}@{}}
\toprule
\textbf{\#} & \textbf{Mechanism} & \textbf{Fallback (where relevant)} & \textbf{Readiness} \\
\midrule
\endfirsthead
\toprule
\textbf{\#} & \textbf{Mechanism} & \textbf{Fallback (where relevant)} & \textbf{Readiness} \\
\midrule
\endhead
\bottomrule
\endlastfoot
T1 & Study author adopts the specification & Author re-specifies geometry & READY \\
T2a & Deterministic geometric gate (S3) & $\le$ 3 attempts via bounded repair loop $\rightarrow$ engineer-authored CAD through the same gate $\rightarrow$ halt & READY \\
T2b & Study author accepts or rejects & Reject $\rightarrow$ T2a fallback path & READY \\
T3a & Scripted mesher: input lineage check, frozen quality rules, conformity check & Halt & READY \\
T3b & Study author judges adequacy & Further refinement, if decided & READY \\
T4a & Deterministic deck writer + independent deck checker & Halt & READY \\
T4b & Study author decides before analysis & — & READY \\
T5a & Pinned solver within a declared validity domain & Halt & READY \\
T5b & Study author judges adequacy & Report limitation & READY \\
T6a & Deterministic parsers; equilibrium and consistency checks & Halt & READY \\
T6b & Study author interprets & — & READY \\
T7a & Study author approves and versions the criteria before any result & Revise as a new version & READY \\
T7b & Deterministic evaluator (accepts only approved criteria and material data) & Halt & READY \\
T8a & Report assembler with numeric-token checker & Halt & READY \\
T8b & Study author accepts or revises & Revise report narrative & READY \\
T9 & Study author decides the response & — & READY \\
T10 & Study author reviews the study package & — & READY \\
\end{longtable}}

\textbf{Scope limits of the records:}

\begin{enumerate}
\item The study is author-led, not independent.
\item Verifier–generator independence is not claimed: generation and acceptance were procedurally separated only (Section~\ref{sec:5.3}), and gate soundness rests on the engineering evidence in S3.3.
\item OS-level isolation of the candidate code was not demonstrated; in-process containment monitoring served as defence in depth only.
\item The tip-link has no demonstrated finite-rotation validity domain, so Q-5 was predeclared ADJUDICATE (S5.4).
\end{enumerate}

\needspace{30\baselineskip}

\section{Type A CAD generation and deterministic gate}\label{sup:S3}

\subsection{Approved geometric formalization}\label{sup:S3.1}

The formalization has \textbf{17 clauses}:

{\renewcommand{\arraystretch}{1.18}
\begin{longtable}{@{}>{\raggedright\arraybackslash}p{\dimexpr 0.1322\linewidth-1\tabcolsep\relax}>{\raggedright\arraybackslash}p{\dimexpr 0.2209\linewidth-2\tabcolsep\relax}>{\raggedright\arraybackslash}p{\dimexpr 0.6469\linewidth-1\tabcolsep\relax}@{}}
\toprule
\textbf{Clause} & \textbf{Family} & \textbf{Condition (abridged)} \\
\midrule
\endfirsthead
\toprule
\textbf{Clause} & \textbf{Family} & \textbf{Condition (abridged)} \\
\midrule
\endhead
\bottomrule
\endlastfoot
E-1 & Execution & Exit code 0 within a 120 s budget \\
E-2 & Execution & Exactly one STEP file, with the declared name \\
U-1 & Units & STEP declares metre as its single length unit \\
U-2 & Units & After calibrated metre import, X span = 2.400 m $\pm$ 1 $\times$ 10$^{-6}$ m \\
V-1 & Validity & Fresh-process re-import yields a non-empty shape \\
V-2 & Validity & BRep validity of the compound and every face \\
T-1 & Topology & Exactly 6 faces, all planar \\
T-2 & Topology & 2 flange, 1 web and 3 stiffener faces, classified by normal; none unclassified \\
T-3 & Topology & Web has 4 circular inner wires; other faces have none \\
G-1 & Geometry & Flange planes at Z = 0.009 and 0.411 m \\
G-2 & Geometry & Flange outlines: (0, $\pm$0.100), (2.4, $\pm$0.060); straight edges \\
G-3 & Geometry & Web outline (0, 0.009)–(2.4, 0.411) \\
G-4 & Geometry & Hole centres; radius 0.060 m; axes $\parallel$ Y \\
G-5 & Geometry & Stiffeners at X = 0.6, 1.2, 1.8 m; outline \{0, 0.050\} $\times$ \{0.009, 0.411\}; one-sided +Y \\
A-1 & Area & Face areas and total area \\
B-1 & Global & Bounding box X [0, 2.4], Y [$-$0.1, 0.1], Z [0.009, 0.411] m \\
P-1 & Provenance & Hashes of program, STEP files and formalization recorded \\
\end{longtable}}

\textbf{Tolerances:}

\begin{itemize}
\item counts and classification: exact;
\item coordinates and radii: 1 $\times$ 10$^{-6}$ m;
\item parallelism: |cos| $\ge$ 0.999999999;
\item areas: 1 $\times$ 10$^{-6}$ relative.
\end{itemize}

The formalization covers specified geometric properties only. Residual adequacy is a separate C[adjudicate] unit (S6.1). The study author approved the formalization before generation.

\needspace{17\baselineskip}

\subsection{Gate checks in the admitted run}\label{sup:S3.2}

The gate executes the candidate under a fixed time budget with in-process containment monitoring, requires one metre-unit STEP file, re-imports it in a fresh process and measures \textbf{23 checks}. All 23 passed:

\par\smallskip\noindent{\renewcommand{\arraystretch}{1.18}%
\begin{tabular}{@{}>{\raggedright\arraybackslash}p{\dimexpr 0.1759\linewidth-1\tabcolsep\relax}>{\raggedright\arraybackslash}p{\dimexpr 0.5035\linewidth-2\tabcolsep\relax}>{\raggedright\arraybackslash}p{\dimexpr 0.1906\linewidth-2\tabcolsep\relax}>{\raggedright\arraybackslash}p{\dimexpr 0.1301\linewidth-1\tabcolsep\relax}@{}}
\toprule
\textbf{Clause} & \textbf{Checks} & \textbf{Tolerance} & \textbf{Result} \\
\midrule
U-1, V-1, V-2 & 3 (unit declaration, non-empty import, BRep validity) & exact & pass \\
T-1, T-2, T-3 & 3 (faces, classification, inner wires) & exact & pass \\
G-1, G-2, G-3 & 4 (flange planes, two flange outlines, web outline) & 1e-6 m & pass \\
G-4 & 3 (hole centres, radii, axes) & 1e-6 m; dot & pass \\
G-5 & 4 (stiffener planes, three outlines) & 1e-6 m & pass \\
A-1 & 4 (flange, web, stiffener, total areas) & 1e-6 rel & pass \\
B-1, U-2 & 2 (bounding box, span sentinel) & 1e-6 m & pass \\
\bottomrule
\end{tabular}}\par\medskip

The remaining clauses were also satisfied:

\begin{itemize}
\item \textbf{E-1:} candidate run time 39.2 s within the 120 s budget;
\item \textbf{E-2:} one STEP file;
\item \textbf{P-1:} hashes recorded.
\end{itemize}

No containment violation was detected. Decision: \textbf{ADMIT}.

\needspace{6\baselineskip}

\subsection{Verifier evidence}\label{sup:S3.3}

The gate's soundness is supported by engineering evidence, not by formal proof. A \textbf{32-case} test suite ran before the study, and all stated expectations were met:

\begin{itemize}
\item \textbf{25 cases with stated expectations:}
\begin{itemize}
\item 4 ADMIT: reference program, alternative construction route, two determinism runs;
\item 18 REJECT: 15 geometric mutations and 3 containment probes;
\item 3 GENERATION\_FAILED: exception, time-out, two STEP files;
\end{itemize}
\item \textbf{7 record-only cases.}
\end{itemize}

The suite was complemented by \textbf{19 logic-level tests} and a review of the gate code by the study author.

\needspace{21\baselineskip}

\textbf{Table S3.3 — Geometric mutations of the engineer-adopted reference program (all REJECT)}

\par\smallskip\noindent{\renewcommand{\arraystretch}{1.18}%
\begin{tabular}{@{}>{\raggedright\arraybackslash}p{\dimexpr 0.6167\linewidth-1\tabcolsep\relax}>{\raggedright\arraybackslash}p{\dimexpr 0.3833\linewidth-1\tabcolsep\relax}@{}}
\toprule
\textbf{Mutation} & \textbf{Clauses failed} \\
\midrule
Hole shifted +1 mm in X & G-4 \\
Hole radius 0.059 m & A-1, G-4 \\
Hole deleted & A-1, G-4, T-3 \\
Stiffener at X = 1.21 m & G-5 \\
Stiffener deleted & A-1, G-5, T-1, T-2 \\
Stiffener on $-$Y side & G-5 \\
Tip flange width 0.130 m & A-1, G-2 \\
Kinked flange, same endpoints & A-1, G-2 \\
Web extended to Z = 0.420 m & A-1, B-1, G-3 \\
Duplicated stiffener face & A-1, G-5, T-1, T-2 \\
STEP truncated after export & U-2, V-1 \\
Correct bounding box printed, hole moved & G-4 \\
Geometry scaled $\times$1000 & A-1, B-1, G-1…G-5, U-2 \\
Thickened solid web & A-1, G-1…G-5, T-1…T-3 \\
Millimetre export of metre geometry & A-1, B-1, G-1…G-5, U-1, U-2 \\
\bottomrule
\end{tabular}}\par\medskip

The containment probes (network access, writing outside the run directory, spawning a child process) were all rejected.

\needspace{6\baselineskip}

\subsection{Bounded repair loop and fallback (predeclared)}\label{sup:S3.4}

\begin{enumerate}
\item At most three sequential AI attempts.
\item Attempts 2–3 receive only the frozen prompt plus the machine-readable failed clauses.
\item After three rejections, an engineer-authored CAD program must pass the same gate.
\item Failing that, the unit halts and the halt is recorded.
\end{enumerate}

\needspace{6\baselineskip}

\subsection{Outcome}\label{sup:S3.5}

The first candidate was \textbf{admitted at the first attempt, with 23/23 checks passed}, and was \textbf{used unmodified}. The repair loop and the fallback were not used.

\emph{Run-order note.} An earlier gate invocation failed before the candidate executed, because of an operator-side path error. It is retained as an invalid run; it did not alter the candidate and consumed no generation attempt.

\needspace{6\baselineskip}

\section{Deterministic downstream workflow}\label{sup:S4}

After admission, no further AI-generated engineering artefact entered the computational chain. Each step below is Type D, and its recorded tool output is the evidence.

\needspace{12\baselineskip}

\subsection{Meshing and refinement}\label{sup:S4.1}

The mesher was a scripted Gmsh 4.15.2 driver. It meshed only the gate-admitted geometry, verified by a lineage check, using homogeneous quadratic triangular shell elements (CalculiX S6 elements) at three refinement levels.

\par\smallskip\noindent{\renewcommand{\arraystretch}{1.18}%
\begin{tabular}{@{}>{\raggedright\arraybackslash}p{\dimexpr 0.1750\linewidth-1\tabcolsep\relax}>{\raggedright\arraybackslash}p{\dimexpr 0.1852\linewidth-2\tabcolsep\relax}>{\raggedright\arraybackslash}p{\dimexpr 0.1848\linewidth-2\tabcolsep\relax}>{\raggedright\arraybackslash}p{\dimexpr 0.1790\linewidth-2\tabcolsep\relax}>{\raggedright\arraybackslash}p{\dimexpr 0.1402\linewidth-2\tabcolsep\relax}>{\raggedright\arraybackslash}p{\dimexpr 0.1358\linewidth-1\tabcolsep\relax}@{}}
\toprule
\textbf{Level} & \textbf{Elements} & \textbf{Quality rules R1–R10} & \textbf{Geometry conformity} & \textbf{min SJ} & \textbf{max aspect} \\
\midrule
Coarse & 3,218 & all pass & pass & 0.775 & 2.22 \\
Medium & 12,152 & all pass & pass & 0.881 & 1.66 \\
Fine & 47,046 & all pass & pass & 0.942 & 1.68 \\
\bottomrule
\end{tabular}}\par\medskip

min SJ: minimum scaled Jacobian; R1–R10: the ten frozen project mesh-quality rules.

Geometry conformity was checked in the CAD kernel, independently of the mesher; the maximum distance was 8.9 $\times$ 10$^{-16}$ m. Repeated meshing on the same host gave identical canonical hashes at every level.

\needspace{6\baselineskip}

\subsection{Solver input decks}\label{sup:S4.2}

A deterministic deck writer produced \textbf{21 decks}:

\begin{itemize}
\item 12 linear static: 3 levels $\times$ 4 load cases;
\item 6 linear buckling: 3 levels $\times$ 2 senses;
\item 2 doubled-load buckling scaling checks (fine mesh);
\item 1 NLGEOM run (fine mesh, positive ultimate load).
\end{itemize}

An independent deterministic checker verified every deck (33–35 checks each). Among other things, it reconstructs the rigid tip load-introduction constraints and confirms that loads and supports are complete. \textbf{Result: 21/21 decks passed, and every deck regenerated byte-identically.} The solver ran only after all decks had passed.

\needspace{6\baselineskip}

\subsection{Solver}\label{sup:S4.3}

CalculiX 2.23 was run as a hash-pinned binary. \textbf{Result: 21/21 jobs completed, with no errors or warnings recorded.} The tool path is supported by the solver's official test suite run on this binary and by project fixtures for bending, buckling and the tip link.

\needspace{14\baselineskip}

\subsection{Extraction and numerical integrity}\label{sup:S4.4}

Results were extracted deterministically from the solver's printed output. Equilibrium is checked in a precision-aware way: each printed value is treated as an interval set by its printed precision. The tolerances are 1 $\times$ 10$^{-6}$ for force and 1 $\times$ 10$^{-4}$ for moment, relative to the applied load.

\par\smallskip\noindent{\renewcommand{\arraystretch}{1.18}%
\begin{tabular}{@{}>{\raggedright\arraybackslash}p{\dimexpr 0.2610\linewidth-1\tabcolsep\relax}>{\raggedright\arraybackslash}p{\dimexpr 0.2416\linewidth-2\tabcolsep\relax}>{\raggedright\arraybackslash}p{\dimexpr 0.2843\linewidth-2\tabcolsep\relax}>{\raggedright\arraybackslash}p{\dimexpr 0.2130\linewidth-1\tabcolsep\relax}@{}}
\toprule
\textbf{Jobs} & \textbf{Force equilibrium} & \textbf{Moment equilibrium} & \textbf{Output consistency} \\
\midrule
12 linear static jobs (3 levels $\times$ 4 load cases) & 12/12 accepted (error $\le$ 5.7 $\times$ 10$^{-12}$) & 12/12 accepted (error $\le$ 6.3 $\times$ 10$^{-7}$) & all conditions satisfied \\
NLGEOM (fine, positive ultimate) & accepted (2.5 $\times$ 10$^{-13}$) & \textbf{not met}: 2.01 $\times$ 10$^{-4}$; diagnostic only, see S5.4 & all conditions satisfied \\
\bottomrule
\end{tabular}}\par\medskip

The fine NLGEOM output files (over 1 GB) were too large for the planned whole-file parsers. The unchanged parsers were applied to byte-exact excerpts of the final increment. On a fixture, the excerpt parse was shown to equal the full-file parse.

\needspace{6\baselineskip}

\subsection{Forward-carrying controls}\label{sup:S4.5}

Three controls carry the admission decision through the chain:

\begin{itemize}
\item the mesher accepts only the gate-admitted geometry;
\item the solver runs only after every deck passes the independent checker;
\item the evaluator accepts only the approved criteria and material data.
\end{itemize}

These controls are independent of the geometry's source by design, although the single run did not test this.

\needspace{21\baselineskip}

\section{Engineering criteria, convergence and dispositions}\label{sup:S5}

\subsection{Frozen criteria and material data}\label{sup:S5.1}

The criteria were approved by the study author and frozen before any result existed. They are project-defined comparison criteria; they are \textbf{not} a CS-25 compliance finding.

\par\smallskip\noindent{\renewcommand{\arraystretch}{1.18}%
\begin{tabular}{@{}>{\raggedright\arraybackslash}p{\dimexpr 0.1417\linewidth-1\tabcolsep\relax}>{\raggedright\arraybackslash}p{\dimexpr 0.2851\linewidth-2\tabcolsep\relax}>{\raggedright\arraybackslash}p{\dimexpr 0.5732\linewidth-1\tabcolsep\relax}@{}}
\toprule
\textbf{Clause} & \textbf{Load level} & \textbf{Criterion} \\
\midrule
Q-1+ / Q-1$-$ & Limit, positive (1.0) / negative ($-$0.4) & Converged $\sigma_{proj,g}$ $\le$ $F_{ty,g}$ for each group g \\
Q-2+ / Q-2$-$ & Ultimate, positive (1.5) / negative ($-$0.6) & Converged $\sigma_{proj,g}$ $\le$ $F_{ty,g}$ (elastic-through-ultimate project criterion) \\
Q-3+ / Q-3$-$ & Buckling, positive / negative (fine mesh) & Converged smallest positive buckling factor $\lambda$ $\ge$ 1.5 \\
Q-4 & All levels & Tip deflection and $\delta$/L reported; no threshold \\
Q-5 & 1.5 $\times$ positive limit load (fine) & Linear vs NLGEOM relative difference $\le$ 0.05; escalation if $\delta$/L > 0.01 or |$\theta$| > 0.05 rad \\
\bottomrule
\end{tabular}}\par\medskip

\textbf{Stress measure.} Von Mises stress is computed from unaveraged integration-point stresses extrapolated to the element corners on both surfaces. The group maximum (rule S1 of the frozen engineering plans) is the largest such value in a component group. $\sigma_{proj}$ (rule S3 of the same plans) is taken outside a predeclared root exclusion zone ($x_{ex}$ = 0.20 m) and is bounded below by the nominal beam-theory reference.

\textbf{Allowables.} MIL-HDBK-5J (2003), Table 3.7.6.0(b1), 7075-T651 plate, A-basis. For the 18 mm flange (0.709 in; 0.500–1.000 in thickness range), the minimum of the L and LT $F_{ty}$ values is 68 ksi $\approx$ \textbf{469 MPa}. For the 8 mm web and stiffeners (0.315 in; 0.250–0.499 in range), the minimum is 67 ksi $\approx$ \textbf{462 MPa}. $F_{tu}$ is context only.

\needspace{6\baselineskip}

\subsection{Convergence rule}\label{sup:S5.2}

Convergence of each quantity $q$ is judged by its relative medium-to-fine change
\[
c(q)=\frac{|q_{fine}-q_{medium}|}{\max(|q_{fine}|,\ q_{floor})},
\]
where $q_{floor}$ is a frozen, quantity-specific floor that prevents division by near-zero values: $10^{-6}$ m for displacement, $10^{6}$ Pa for stress and $10^{-3}$ for buckling factors. In every reported case $|q_{fine}|$ exceeds its floor (for stresses by a factor of at least about 30; for displacements and buckling factors by several orders of magnitude), so the floor did not affect any reported value.

Thresholds were frozen before the run:

\begin{itemize}
\item tip displacement: c $\le$ 0.01;
\item buckling factors: c $\le$ 0.01;
\item stress metrics: c $\le$ 0.05.
\end{itemize}

A quantity that fails is \textbf{NOT\_CONVERGED}, and the threshold is never relaxed.

\emph{Worked example (flange $\sigma_{proj}$, positive ultimate):} c = |246.54 $-$ 182.60| / 246.54 = 0.259 > 0.05, so NOT\_CONVERGED.

\needspace{15\baselineskip}

\subsection{Results}\label{sup:S5.3}

The relative changes agree across load cases to within rounding (tip displacement 9.1–9.2e-5; stresses identical to the precision shown), consistent with the linear analysis.

\textbf{Table S5.3a — Tip displacement (mm)}

\par\smallskip\noindent{\renewcommand{\arraystretch}{1.18}%
\begin{tabular}{@{}>{\raggedright\arraybackslash}p{\dimexpr 0.1813\linewidth-1\tabcolsep\relax}>{\raggedright\arraybackslash}p{\dimexpr 0.1546\linewidth-2\tabcolsep\relax}>{\raggedright\arraybackslash}p{\dimexpr 0.1546\linewidth-2\tabcolsep\relax}>{\raggedright\arraybackslash}p{\dimexpr 0.1546\linewidth-2\tabcolsep\relax}>{\raggedright\arraybackslash}p{\dimexpr 0.1162\linewidth-2\tabcolsep\relax}>{\raggedright\arraybackslash}p{\dimexpr 0.2388\linewidth-1\tabcolsep\relax}@{}}
\toprule
\textbf{Load case} & \textbf{Coarse} & \textbf{Medium} & \textbf{Fine} & \textbf{c} & \textbf{Status} \\
\midrule
Limit + (1.0) & 20.8968 & 20.9006 & 20.9025 & 9.1e-5 & CONVERGED \\
Ultimate + (1.5) & 31.3452 & 31.3510 & 31.3538 & 9.1e-5 & CONVERGED \\
Limit $-$ ($-$0.4) & $-$8.3587 & $-$8.3603 & $-$8.3610 & 9.1e-5 & CONVERGED \\
Ultimate $-$ ($-$0.6) & $-$12.5381 & $-$12.5404 & $-$12.5415 & 9.2e-5 & CONVERGED \\
\bottomrule
\end{tabular}}\par\medskip

\needspace{8\baselineskip}

\textbf{Table S5.3b — Smallest positive buckling factor}

\par\smallskip\noindent{\renewcommand{\arraystretch}{1.18}%
\begin{tabular}{@{}>{\raggedright\arraybackslash}p{\dimexpr 0.1541\linewidth-1\tabcolsep\relax}>{\raggedright\arraybackslash}p{\dimexpr 0.1582\linewidth-2\tabcolsep\relax}>{\raggedright\arraybackslash}p{\dimexpr 0.1616\linewidth-2\tabcolsep\relax}>{\raggedright\arraybackslash}p{\dimexpr 0.1616\linewidth-2\tabcolsep\relax}>{\raggedright\arraybackslash}p{\dimexpr 0.1191\linewidth-2\tabcolsep\relax}>{\raggedright\arraybackslash}p{\dimexpr 0.2455\linewidth-1\tabcolsep\relax}@{}}
\toprule
\textbf{Sense} & \textbf{Coarse} & \textbf{Medium} & \textbf{Fine} & \textbf{c} & \textbf{Status} \\
\midrule
Positive & 4.852521 & 4.839756 & 4.834881 & 1.0e-3 & CONVERGED \\
Negative & 12.1364 & 12.09921 & 12.08619 & 1.1e-3 & CONVERGED \\
\bottomrule
\end{tabular}}\par\medskip

\needspace{17\baselineskip}

\textbf{Table S5.3c — Flange stress (MPa)}

The raw maximum lies at the clamped root corner and is reported separately as a non-converged local peak; it is not removed.

\par\smallskip\noindent{\renewcommand{\arraystretch}{1.18}%
\begin{tabular}{@{}>{\raggedright\arraybackslash}p{\dimexpr 0.1361\linewidth-1\tabcolsep\relax}>{\raggedright\arraybackslash}p{\dimexpr 0.1525\linewidth-2\tabcolsep\relax}>{\raggedright\arraybackslash}p{\dimexpr 0.1078\linewidth-2\tabcolsep\relax}>{\raggedright\arraybackslash}p{\dimexpr 0.1250\linewidth-2\tabcolsep\relax}>{\raggedright\arraybackslash}p{\dimexpr 0.0960\linewidth-2\tabcolsep\relax}>{\raggedright\arraybackslash}p{\dimexpr 0.0975\linewidth-2\tabcolsep\relax}>{\raggedright\arraybackslash}p{\dimexpr 0.2851\linewidth-1\tabcolsep\relax}@{}}
\toprule
\textbf{Load case} & \textbf{Quantity} & \textbf{Coarse} & \textbf{Medium} & \textbf{Fine} & \textbf{c} & \textbf{Status} \\
\midrule
Limit + & raw maximum & 170.1 & 199.5 & 226.8 & 0.120 & reported separately \\
Limit + & $\sigma_{proj}$ & — & 121.7 & 164.4 & 0.259 & NOT\_CONVERGED \\
Ultimate + & raw maximum & 255.1 & 299.3 & 340.2 & 0.120 & reported separately \\
Ultimate + & $\sigma_{proj}$ & — & 182.6 & 246.5 & 0.259 & NOT\_CONVERGED \\
Limit $-$ & raw maximum & 68.0 & 79.8 & 90.7 & 0.120 & reported separately \\
Limit $-$ & $\sigma_{proj}$ & — & 48.7 & 65.7 & 0.259 & NOT\_CONVERGED \\
Ultimate $-$ & raw maximum & 102.0 & 119.7 & 136.1 & 0.120 & reported separately \\
Ultimate $-$ & $\sigma_{proj}$ & — & 73.0 & 98.6 & 0.259 & NOT\_CONVERGED \\
\bottomrule
\end{tabular}}\par\medskip

\needspace{10\baselineskip}

\textbf{Table S5.3d — Web and stiffener group-maximum stress (rule S1; MPa)}

\par\smallskip\noindent{\renewcommand{\arraystretch}{1.18}%
\begin{tabular}{@{}>{\raggedright\arraybackslash}p{\dimexpr 0.1852\linewidth-1\tabcolsep\relax}>{\raggedright\arraybackslash}p{\dimexpr 0.2679\linewidth-2\tabcolsep\relax}>{\raggedright\arraybackslash}p{\dimexpr 0.1281\linewidth-2\tabcolsep\relax}>{\raggedright\arraybackslash}p{\dimexpr 0.2593\linewidth-2\tabcolsep\relax}>{\raggedright\arraybackslash}p{\dimexpr 0.1595\linewidth-1\tabcolsep\relax}@{}}
\toprule
\textbf{Load case} & \textbf{Web (coarse / medium / fine)} & \textbf{Web c} & \textbf{Stiffener (coarse / medium / fine)} & \textbf{Stiffener c} \\
\midrule
Limit + & 119.1 / 118.4 / 122.6 & 0.035 & 51.8 / 65.6 / 78.5 & 0.164 \\
Ultimate + & 178.7 / 177.5 / 183.9 & 0.035 & 77.6 / 98.4 / 117.8 & 0.164 \\
Limit $-$ & 47.7 / 47.3 / 49.1 & 0.035 & 20.7 / 26.2 / 31.4 & 0.164 \\
Ultimate $-$ & 71.5 / 71.0 / 73.6 & 0.035 & 31.1 / 39.4 / 47.1 & 0.164 \\
\bottomrule
\end{tabular}}\par\medskip

Web: CONVERGED. Stiffener: NOT\_CONVERGED; its peak lies at a stiffener–flange junction, and junctions are included in the metric by design.

\textbf{Nominal root references:} flange 114.2 / 171.3 / 45.7 / 68.5 MPa, and web 34.1 / 51.2 / 13.6 / 20.5 MPa (Limit + / Ultimate + / Limit $-$ / Ultimate $-$).

\needspace{14\baselineskip}

\subsection{Clause verdicts}\label{sup:S5.4}

\textbf{Table S5.4a — Stress clauses}

The ratio $\sigma_{proj}$ / $F_{ty}$ is shown for information only. A non-converged metric cannot yield PASS.

\par\smallskip\noindent{\renewcommand{\arraystretch}{1.18}%
\begin{tabular}{@{}>{\raggedright\arraybackslash}p{\dimexpr 0.1423\linewidth-1\tabcolsep\relax}>{\raggedright\arraybackslash}p{\dimexpr 0.1453\linewidth-2\tabcolsep\relax}>{\raggedright\arraybackslash}p{\dimexpr 0.1316\linewidth-2\tabcolsep\relax}>{\raggedright\arraybackslash}p{\dimexpr 0.1631\linewidth-2\tabcolsep\relax}>{\raggedright\arraybackslash}p{\dimexpr 0.4177\linewidth-1\tabcolsep\relax}@{}}
\toprule
\textbf{Clause} & \textbf{Flange} & \textbf{Web} & \textbf{Stiffener} & \textbf{Evaluator verdict} \\
\midrule
Q-1+ & 0.350 & 0.265 & 0.170 & NOT\_CONVERGED \\
Q-1$-$ & 0.140 & 0.106 & 0.068 & NOT\_CONVERGED \\
Q-2+ & 0.526 & 0.398 & 0.255 & NOT\_CONVERGED \\
Q-2$-$ & 0.210 & 0.159 & 0.102 & NOT\_CONVERGED \\
\bottomrule
\end{tabular}}\par\medskip

\needspace{11\baselineskip}

\textbf{Table S5.4b — Other clauses}

\par\smallskip\noindent{\renewcommand{\arraystretch}{1.18}%
\begin{tabular}{@{}>{\raggedright\arraybackslash}p{\dimexpr 0.1329\linewidth-1\tabcolsep\relax}>{\raggedright\arraybackslash}p{\dimexpr 0.6295\linewidth-2\tabcolsep\relax}>{\raggedright\arraybackslash}p{\dimexpr 0.2377\linewidth-1\tabcolsep\relax}@{}}
\toprule
\textbf{Clause} & \textbf{Values} & \textbf{Evaluator verdict} \\
\midrule
Q-3+ & $\lambda$ = 4.834881 $\ge$ 1.5; run complete; converged & PASS \\
Q-3$-$ & $\lambda$ = 12.08619 $\ge$ 1.5; run complete; converged & PASS \\
Q-4 & $\delta$/L = 0.0087, 0.0131, $-$0.0035, $-$0.0052 (Limit +, Ultimate +, Limit $-$, Ultimate $-$) & REPORTED \\
Q-5 & r($\delta_{tip}$) 2.4e-4; r($\sigma$) flange 0.0196, web 0.0075, stiffener 3.4e-4 & SUPPORTED \\
\bottomrule
\end{tabular}}\par\medskip

\textbf{Buckling diagnostics (not a verdict condition).}

\begin{itemize}
\item The doubled-load scaling check raised no anomaly (e = 3.2 $\times$ 10$^{-7}$ and 4.9 $\times$ 10$^{-7}$).
\item Each fine-mesh buckling run returned 20 positive factors.
\end{itemize}

\textbf{Q-5.} Before any analysis, Q-5 was declared to require adjudication: the finite-rotation validity of the linearized tip-load relation had not been established. The results support that caution:

\begin{itemize}
\item $\delta$/L at ultimate load = 0.0131, above the 0.01 escalation threshold;
\item the NLGEOM run showed a moment-equilibrium deficit (2.01 $\times$ 10$^{-4}$) consistent with the known limitation of the linearized link (S6.3).
\end{itemize}

\needspace{11\baselineskip}

\subsection{Final dispositions}\label{sup:S5.5}

\par\smallskip\noindent{\renewcommand{\arraystretch}{1.18}%
\begin{tabular}{@{}>{\raggedright\arraybackslash}p{\dimexpr 0.1628\linewidth-1\tabcolsep\relax}>{\raggedright\arraybackslash}p{\dimexpr 0.4390\linewidth-2\tabcolsep\relax}>{\raggedright\arraybackslash}p{\dimexpr 0.3982\linewidth-1\tabcolsep\relax}@{}}
\toprule
\textbf{Clause} & \textbf{Evaluator verdict} & \textbf{Final disposition} \\
\midrule
Q-1$\pm$ & NOT\_CONVERGED & ADJUDICATE \\
Q-2$\pm$ & NOT\_CONVERGED & ADJUDICATE \\
Q-3$\pm$ & PASS & PASS \\
Q-4 & REPORTED & Not applicable \\
Q-5 & SUPPORTED & ADJUDICATE (predeclared) \\
\bottomrule
\end{tabular}}\par\medskip

Favourable stress ratios (at most 0.526) did \textbf{not} yield PASS, because the required convergence evidence was absent. Those clauses received an ADJUDICATE disposition, and no threshold or criterion was changed.

\needspace{13\baselineskip}

\subsection{Unit outcomes: READY is not PASS}\label{sup:S5.6}

Unit outcomes are study-level outcome labels, distinct from clause verdicts, deployment readiness and the regime guarantees. Over the 17 units, all READY before the run:

\par\smallskip\noindent{\renewcommand{\arraystretch}{1.18}%
\begin{tabular}{@{}>{\raggedright\arraybackslash}p{\dimexpr 0.4319\linewidth-1\tabcolsep\relax}>{\raggedright\arraybackslash}p{\dimexpr 0.1182\linewidth-2\tabcolsep\relax}>{\raggedright\arraybackslash}p{\dimexpr 0.4499\linewidth-1\tabcolsep\relax}@{}}
\toprule
\textbf{Outcome} & \textbf{Count} & \textbf{Units} \\
\midrule
PASS & 12 & T1, T2a, T2b, T3a, T4a, T5a, T6a, T6b, T7b, T8a, T8b, T10 \\
ADJUDICATE & 3 & T3b, T5b, T9 \\
Not applicable (pre-run decisions) & 2 & T4b, T7a \\
FAIL & 0 & — \\
\bottomrule
\end{tabular}}\par\medskip

The evaluation unit (T7b) recorded PASS because it executed the approved criteria as specified, even though several of the dispositions it produced were ADJUDICATE.

\needspace{33\baselineskip}

\section{Human judgement and additional diagnostics}\label{sup:S6}

\subsection{Engineering judgements made by the study author}\label{sup:S6.1}

{\renewcommand{\arraystretch}{1.18}
\begin{longtable}{@{}>{\raggedright\arraybackslash}p{\dimexpr 0.2156\linewidth-1\tabcolsep\relax}>{\raggedright\arraybackslash}p{\dimexpr 0.1719\linewidth-2\tabcolsep\relax}>{\raggedright\arraybackslash}p{\dimexpr 0.6125\linewidth-1\tabcolsep\relax}@{}}
\toprule
\textbf{Judgement} & \textbf{Regime} & \textbf{Decision} \\
\midrule
\endfirsthead
\toprule
\textbf{Judgement} & \textbf{Regime} & \textbf{Decision} \\
\midrule
\endhead
\bottomrule
\endlastfoot
Residual CAD adequacy & C[adjudicate] & The study author judged the admitted geometry adequate for the study and accepted it with limitations. Defensive code branches unnecessary in the pinned environment were recorded rather than edited, because the predeclared protocol allowed modification only through the repair loop. \\
Formalization of acceptance criteria & C[adjudicate] & The study author approved the criteria (S5.1) before any result. \\
Discretization adequacy & C[assist] & The study author accepted the three-level study with limitations. Tip displacement, buckling factors and the web stress converged; flange $\sigma_{proj}$ and the stiffener did not. No further refinement was made within this study. \\
Analysis adequacy and Q-5 & C[assist] & The study author kept Q-5 at ADJUDICATE per the predeclared rule and judged the NLGEOM moment deficit consistent with $\theta$$^{2}$/2 (S6.3). It was not used as release-supporting evidence. \\
Result interpretation & C[assist] & The study author identified the first buckling mode in both senses as a global lateral-torsional mode of the laterally unrestrained cantilever, so the factors depend on the idealization of tip and lateral restraint. \\
Response to unresolved stress clauses & C[assist] & Recorded as a limitation, with no design change. Model refinement was deferred to a new versioned iteration, and no criterion was changed. \\
Study-level review & C[assist] & The demonstration was accepted with limitations as a research demonstration, not a certification. \\
\end{longtable}}

\needspace{13\baselineskip}

\subsection{Post-hoc diagnostics (not verdict inputs)}\label{sup:S6.2}

These were computed \textbf{after} the verdicts were recorded, at positive ultimate load. They were not passed to the evaluator, changed no verdict, and led to no change in any criterion, formalization, metric, tolerance or exclusion zone.

\par\smallskip\noindent{\renewcommand{\arraystretch}{1.18}%
\begin{tabular}{@{}>{\raggedright\arraybackslash}p{\dimexpr 0.5206\linewidth-1\tabcolsep\relax}>{\raggedright\arraybackslash}p{\dimexpr 0.1933\linewidth-2\tabcolsep\relax}>{\raggedright\arraybackslash}p{\dimexpr 0.1642\linewidth-2\tabcolsep\relax}>{\raggedright\arraybackslash}p{\dimexpr 0.1219\linewidth-1\tabcolsep\relax}@{}}
\toprule
\textbf{Diagnostic} & \textbf{Medium} & \textbf{Fine} & \textbf{c} \\
\midrule
Interior flange, 0.20 m $\le$ x $\le$ 2.20 m & 179.4 MPa & 179.7 MPa & 0.2 \% \\
Stiffener, excluding elements within 0.05 m of the flanges & 41.9 MPa & 48.2 MPa & 13.0 \% \\
\bottomrule
\end{tabular}}\par\medskip

The flange result suggests that the flange non-convergence is confined to the end regions: the root clamp and the tip load-introduction corner. The stiffener remains unresolved. These diagnostics help explain the outcome; they do not establish compliance.

\needspace{6\baselineskip}

\subsection{Finite-rotation behaviour of the tip link}\label{sup:S6.3}

At positive ultimate load the tip rotation is $\theta$ $\approx$ 0.0202 rad. The NLGEOM moment deficit (2.01 $\times$ 10$^{-4}$) matches $\theta$$^{2}$/2 = 2.05 $\times$ 10$^{-4}$. This is the deficit of a linearized rigid-link moment under finite rotation, consistent with the predeclared limitation that the link has no demonstrated finite-rotation validity domain.

\needspace{6\baselineskip}

\subsection{Formalization limitation observed in execution}\label{sup:S6.4}

The predeclared stress rule represented the root-clamp region but not the tip load-introduction region revealed by the run. This was recorded as a limitation. Any refinement would be a new version of the formalization, fixed before a future rerun, not adopted retrospectively.

\needspace{6\baselineskip}

\section{Retrospective framework audit of the original case-study record}\label{sup:S7}

Before the controlled execution, the framework was applied retrospectively to the \textbf{original} wing-spar case-study record (the record of an earlier execution of the same analysis, from April 2026, made before the controlled protocol was defined), using only the implementation and evidence that actually existed. The retrospective audit identified \textbf{22 auditable items} in the original case-study record.

\textbf{This 22-item retrospective audit is not interchangeable with the 17-workflow-unit controlled execution (S1–S6).} The decomposition, the evidence base and the execution differ, so the two sets must not be pooled or compared as rates.

\needspace{21\baselineskip}

\textbf{Table S7.1 — Retrospective regime allocation (22 audited items)}

\par\smallskip\noindent{\renewcommand{\arraystretch}{1.18}%
\begin{tabular}{@{}>{\raggedright\arraybackslash}p{\dimexpr 0.2576\linewidth-1\tabcolsep\relax}>{\raggedright\arraybackslash}p{\dimexpr 0.7424\linewidth-1\tabcolsep\relax}@{}}
\toprule
\textbf{Activity} & \textbf{Audited items and retrospective regime} \\
\midrule
Geometry definition & T1a numeric specification: C[adjudicate]; T1b concept and sizing: C[adjudicate] \\
CAD generation & T2a CAD program generation: C[adjudicate]; T2b residual CAD adequacy: C[adjudicate] \\
Meshing & T3a mesh generation: C[adjudicate]; T3b discretization adequacy: C[assist] \\
Loads and boundary conditions & T4a card values: C[adjudicate]; T4b load cases and idealization: C[assist] \\
FE solution & T5a solver execution: D; T5b analysis adequacy: C[assist] \\
Result extraction & T6a eigenvalue/load extraction: C[adjudicate]; T6b peak extraction: C[adjudicate]; T6c interpretation: C[assist] \\
Compliance & T7a formalization: C[adjudicate]; T7b evaluation: D \\
Engineering report & T8a numerical consistency: C[adjudicate]; T8b report narrative: C[adjudicate] \\
Design modification & T9a option estimates: C[adjudicate]; T9b decision: C[assist]; T9c criterion change: C[adjudicate] \\
Final approval & T10a numeric restatement: C[adjudicate]; T10b final authorization: C[assist] \\
\bottomrule
\end{tabular}}\par\medskip

\textbf{Tally:} C[adjudicate] \textbf{14}, C[assist] \textbf{6}, D \textbf{2}, A \textbf{0}, B \textbf{0}, E \textbf{0}. \textbf{Readiness: 0/22.}

\textbf{What the audit shows.} Checks that the original case-study record had presented as formal or statistical verification were, on the surviving evidence, unrecorded human adjudication. The framework classified them accordingly rather than forcing Type A or Type B. Several audited items have a Type A or D \textbf{target regime}, for example:

\begin{itemize}
\item an implemented geometric verifier for CAD generation;
\item a validated parser for result extraction;
\item a deterministic numerical cross-check for report numbers.
\end{itemize}

These targets motivated the controlled execution.

\textbf{What NOT READY means here.} NOT READY means that the retrospective evidence package does not satisfy every obligation of the assigned record: identified authorization, AI provenance, approved formalizations or tool validation evidence (where an implemented verifier was missing, this determined the classification and a Type A target regime, not readiness). It does \textbf{not} mean that the engineering work was unsafe or wrong, or that a deterministic tool was unreliable. For example, the original solver run completed with machine-precision equilibrium, yet its record is NOT READY because validation evidence and a declared validity domain were not cited.

\end{document}